\documentclass[lettersize,journal]{IEEEtran}
\usepackage{amsmath,amsfonts}
\usepackage{algorithmic}
\usepackage{algorithm}
\usepackage{array}
\usepackage{textcomp}
\usepackage{stfloats}
\usepackage{url}
\usepackage{verbatim}
\usepackage{graphicx}
\usepackage{cite}
\usepackage{nicefrac}
\usepackage{amsthm}
\usepackage{multirow}
\usepackage{bm}
\usepackage{subcaption}
\usepackage{amssymb}
\usepackage{booktabs}
\usepackage{float}
\usepackage{colortbl}
\usepackage[dvipsnames]{xcolor}
\newtheorem*{remark}{Remark}
\usepackage{tikz}
\usetikzlibrary{positioning}
\usetikzlibrary{arrows.meta}
\usepackage{makecell}
\usepackage{pifont}
\newcommand{\cmark}{\ding{51}} 
\newcommand{\xmark}{\ding{55}} 

\definecolor{navy}{RGB}{0,0,128} %
\definecolor{highlightgray}{gray}{0.95}
\definecolor{headergray}{gray}{0.92}

\newcommand{\best}[1]{\textbf{\textcolor{BrickRed}{#1}}}
\newcommand{\second}[1]{\underline{\textcolor{RoyalBlue}{#1}}}

\begin{document}

\title{CoFi-UCGen: Coarse-to-Fine Unsupervised Conditional Generation without Label Priors}

 \author{Shengxi~Li,~\IEEEmembership{Member,~IEEE,}
        Zhaokun~Hu,~\IEEEmembership{Student~Member,~IEEE,}
        Ce~Zheng,
        Mai~Xu,~\IEEEmembership{Senior~Member,~IEEE,}
        Jingyuan Xia, ~\IEEEmembership{Member,~IEEE,}
        Si Liu,~\IEEEmembership{Member,~IEEE,} 
 		\IEEEcompsocitemizethanks{
    \IEEEcompsocthanksitem M. Xu and C. Zheng are the corresponding authors (e-mail: maixu@buaa.edu.cn, zhengce@buaa.edu.cn).
    \IEEEcompsocthanksitem S. Li, Z. Hu, M. Xu are with the Department of Electronic Information Engineering, Beihang University, Beijing 100191, China (e-mail: lishengxi@buaa.edu.cn,  	ZhaokunHu@buaa.edu.cn).
    \IEEEcompsocthanksitem C. Zheng is with the School of Cyber Science and Technology, Beihang University, Beijing 100191, China.
    \IEEEcompsocthanksitem J. Xia is with the College of Electronic Science, National University of Defense Technology, Hunan 410073, China (e-mail: j.xia10@nudt.edu.cn).
    \IEEEcompsocthanksitem S. Liu is with the Institute of Artificial Intelligence, Beihang University, Beijing 100191, China, (e-mail: liusi@buaa.edu.cn).
}
}

\markboth{Journal of \LaTeX\ Class Files,~Vol.~14, No.~8, August~2021}%
{Shell \MakeLowercase{\textit{et al.}}: A Sample Article Using IEEEtran.cls for IEEE Journals}


\maketitle

\begin{abstract}
Unsupervised conditional image generation (UCGen) aims to control generation without relying on manually annotated labels, yet remains challenging due to unstructured semantic representations across granularities. To address this,  we propose a novel coarse-to-fine UCGen framework (CoFi-UCGen) that explicitly disentangles global semantics from fine-grained variations, which to the best of our knowledge, sets out the first successful attempt for both coarse- and fine-grained conditional generation without any labels. More specifically, we first propose the adversarial semantic reciprocal learning theory to ensure the semantic consistency and completeness between images and latent spaces. Based on the consistency, we propose the bit-codes to learn a structured coarse-grained latent space, and further prove distinct global semantics inherent from our bit-codes while preserving independent noise sampling for generation. Building upon these bit-codes, we establish a fine-grained semantic basis and introduce a hierarchical modulation mechanism in diffusion models, by enabling layer-wise injection from coarse conditions to progressively control fine-grained attributes during generation. Extensive experiments demonstrate that without any label priors or pre-trained feature extractors, our CoFi-UCGen consistently outperforms existing UCGen methods in terms of image quality, semantic consistency, and control accuracy, verifying the effectiveness of explicit coarse-to-fine semantic decomposition for the challenging UCGen task.
\end{abstract}

\begin{IEEEkeywords}
Unsupervised conditional generation, diffusion model modulation, reciprocal learning, representation learning.
\end{IEEEkeywords}

\section{Introduction}

\IEEEPARstart{T}{he} recent advances of deep generative models (DGMs), witnessing evolutions from variational autoencoders, generative adversarial networks (GANs) and diffusion models, allowing for high-quality and realistic image generation from noisy latent space \cite{kingma2013auto, goodfellow2014generative, rombach2022high}. The primal generative models can only generate images from purely random noise, and the way to control the generation may rely on the post-processing within the latent space to demystify semantic directions \cite{harkonen2020ganspace, kwon2023diffusion, proszewska2024multi}. On the other hand, conditional deep generative models (cDGMs) facilitate controllable generation despite intrinsic stochasticity, by learning a continuous latent space conditioned on label priors from pre-defined auxiliary information \cite{mirza2014conditional, odena2017conditional, dhariwal2021diffusion, li2023neural}. The majority of cDGMs focus on the categorical conditions, in which the auxiliary information accounts for classes and attributes \cite{kaneko2017generative, lu2018attribute, hou2022conditional}, with the most recent extensions to other label forms including texts \cite{dash2017tac, liao2022text, zhang2024motiondiffuse}, styles \cite{xu2021drb, wang2023stylediffusion}, sketches \cite{chen2018sketchygan, lu2018image}, to name but a few.

\begin{figure}[t]
    \centering
    \includegraphics[width=1\linewidth]{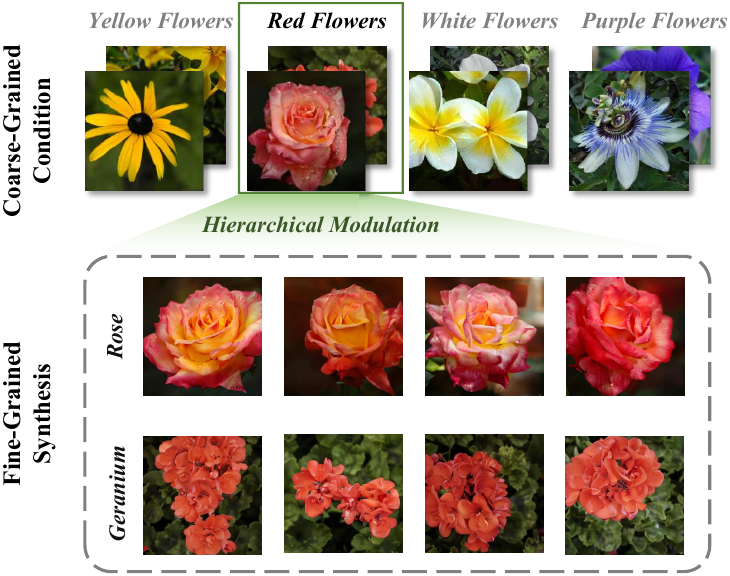}
    \caption{Illustration of our CoFi-UCGen method. \textit{The top panel} visualizes the coarse-grained semantic space, whereby images can be interpreted as coarse-grained clusters (e.g., red flowers). By anchoring on a specific coarse condition (highlighted in the green box), our model employs hierarchical modulation to guide the diffusion process. \textit{The bottom panel} demonstrates the fine-grained controllability, in which the model generates diversifying high-fidelity instances with distinct fine-grained attributes (e.g., Rose vs. Geranium), while strictly adhering to the global semantic structure defined by the coarse condition.}
    \label{fig1}
    \vspace{-2em}
\end{figure}

Despite success, training cDGMs is tightly related to the quality of conditional datasets, in which the manually annotated labels are oftentimes costly and time-consuming. The existing label information, either in coarse categorical or fine-grained text formats, may also be insufficient to summarise the semantics within an image, which typically require label balancing \cite{conditional1_1, conditional1_2, conditional1_3} and particular ways of embedding the auxiliary information \cite{li2023neural}. To address this inherent limitations of pre-defined labels, a new path route of unsupervised conditional generation (UCGen) is investigated to conditionally generate images based on underlying structures or implicit representations, instead of the ground-truth labels; this effectively avoids relying on manually annotated labels. UCGen promises additional benefits in extensive applications including domain adaptation by manipulating latent factors and multi-modal exploration from semantically structured latent spaces.

To address the absence of ground-truth labels, a straightforward strategy is to obtain pseudo-labels in advance from pre-trained self-supervised visual feature extractors, followed by training regular cDGMs via those pseudo-labels as auxiliary conditions, e.g., via SimCLR \cite{chen2020simple} or DINO \cite{caron2021emerging, oquab2023dinov2}. Despite the distinct semantics reflected by pseudo-labels, the generation randomness, arising from the latent noise, is essentially irrelevant to those pre-trained pseudo-labels. The irrelevance leads to the ambiguity in low-density or anisotropic areas for conditional generation. Several methods, by investigating latent spaces of pre-trained DGMs, can also achieve UCGen functionally \cite{harkonen2020ganspace}; this, however, typically assumes the orthogonal semantic directions from principle components, which falls short in handling fine-grained complex scenarios or high-dimensional spaces from diffusion models \cite{tzelepis2021warpedganspace, kwon2023diffusion}. Furthermore, few UCGen works aggregate the auxiliary information in an end-to-end optimisation style when training DGMs, by relying on either concatenating learnable pseudo-labels \cite{mukherjee2019clustergan, liu2020diverse} or grouping the latent space in pre-defined Gaussian mixture models (GMMs) \cite{GMMGAN, ying2021unsupervised}, thus allowing for clustered generation to accommodate distinct semantic attributes during optimization. This essentially constrains latent spaces by separate clusters, thus suffering from training instability and deficiency on generation quality. Furthermore, given limited numbers of clusters, the above methods encounter insufficient capability when representing rich and distinct semantics.

Indeed, apart from clustering the latent space, the consecutive layers from generators are inherently hierarchical \cite{karras2019style, baugan, baranchuklabel}, thus capable of reflecting varying levels of distinct semantics that govern the granularity of the finally generated images. It is thus beneficial to control the generation by carefully re-organising the latent space and generating layers, corresponding to the hierarchy from the latent space and generators. This allows us even to achieve a novel coarse-to-fine UCGen method (CoFi-UCGen) that simultaneously controls the coarse-grained and fine-grained semantics without any label priors, as illustrated in Fig. \ref{fig1}. More specifically, our CoFi-UCGen benefits from both GANs and diffusion models, namely, the compact representation from GANs for coarse semantics and detail refinement from diffusion models for fine-grained semantics. Regarding the coarse semantics, instead of the separated latent space by clusters, we first propose a novel format of bit-codes in the noisy generation space of GANs, that retains independent and identically distributed (\textit{i.i.d}) noise sampling for disentangled and complete semantics. We also prove theoretically that each bit in our bit-codes corresponds to a valid semantic direction. Different from the pre-defined pseudo-labels in existing methods, our bit-codes, seamlessly incorporated as the generating noise, are optimised from scratch by our adversarial reciprocal learning theory, capable of representing rich, accurate and distinct semantics.

To achieve fine-grained control on semantics, we further propose the hierarchical modulation generator (termed HM-UNet) under the architecture of diffusion models, which automatically learns a novel semantic basis corresponding to the generator layers upon our bit-codes. Experimental results have verified that given the absence of manual annotation, our CoFi-UCGen method is able to demystify the intrinsic labels within images, thus achieving superior performances on both quality and accuracy for unsupervised conditional generation. Besides superiority on performances, the proposed CoFi-UCGen offers unprecedented control over the synthesis process across both coarse- and fine-grained granularities. Our main contributions are three-fold:
\begin{itemize}
    \item We introduce the new bit-codes for the coarse semantics when learning UCGen, which retains \textit{i.i.d} property to establish a valid and continuous generating space.
    \item We propose the adversarial reciprocal learning theory that is able to represent rich, accurate and distinct semantics within our bit-codes, also with theoretical proofs.
    \item We propose the hierarchical modulation generation (realized by our HM-UNet) upon diffusion models, bridging the coarse bit-codes and fine-grained semantic basis for coarse-to-fine UCGen.
\end{itemize}

\begin{figure}[t]
    \centering
    \includegraphics[width=0.9\linewidth]{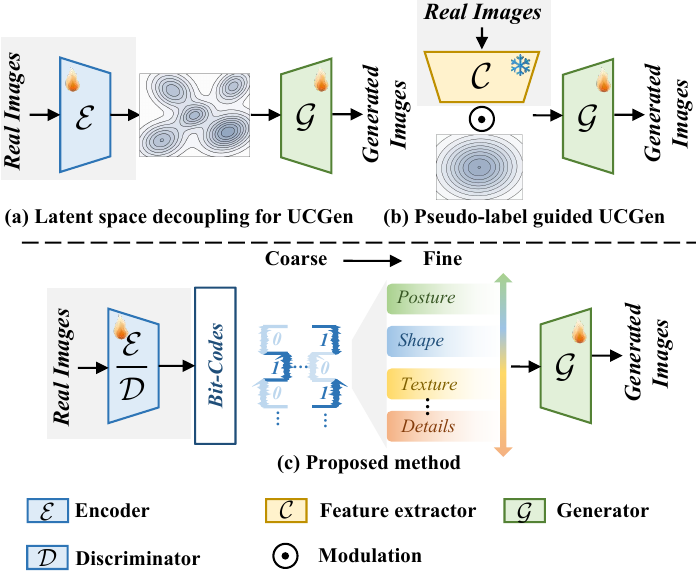}
    \caption{Illustration of existing paradigms in UCGen, against our CoFi-UCGen framework. $\mathcal{E}$ denotes the encoder, $\mathcal{G}$ represents the generator, and $\mathcal{C}$ denotes the pre-trained feature extractor. Components shaded in gray are employed only for training. Please note that the two paradigms represented by (a) and (b) can only perform coarse-grained UCGen, whilst suffering deficiency on generation quality and control accuracy due to low-density areas or irrelevant pseudo-labels. In contrast, our CoFi-UCGen framework can achieve decoupled and controllable generation with both coarse and fine granularity, achieving high-quality and accurate UCGen performances.}
    \label{fig2_paradigm}
    \vspace{-1.5em}
\end{figure}

\begin{figure*}[t]
    \centering
    \includegraphics[width=\linewidth]{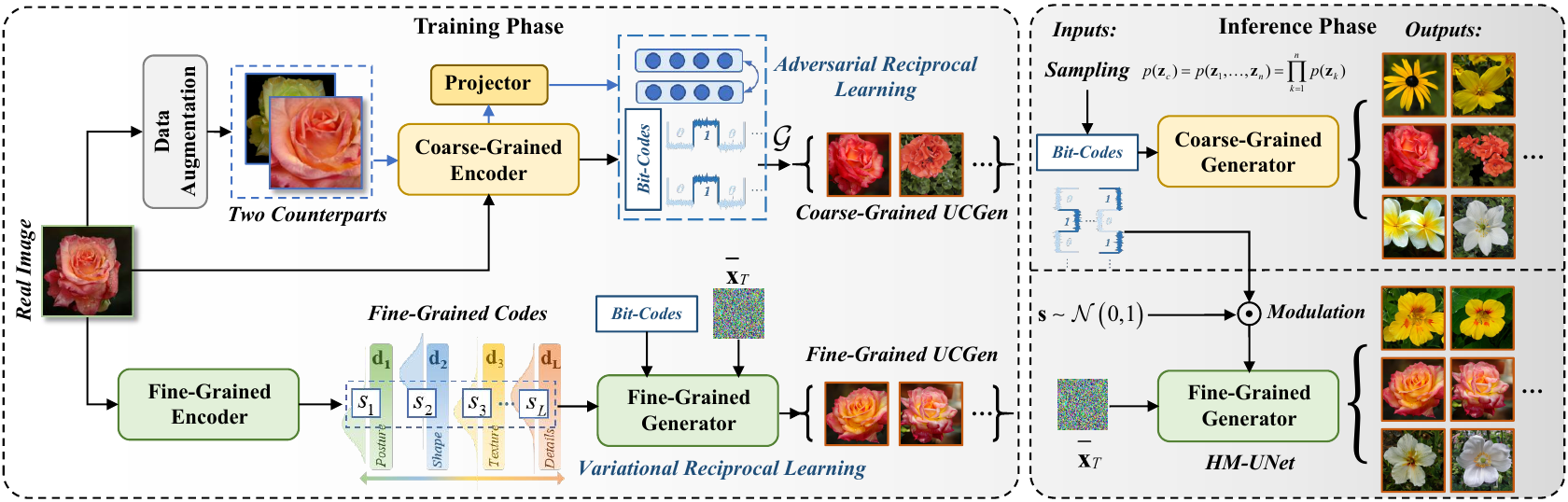}
    \caption{{Overview of the proposed CoFi-UCGen framework. At the \textit{Training Phase}, our CoFi-UCGen does not require any label prior but investigates the intrinsic semantic structures within generative models. The top panel illustrates the coarse-grained stage, in which two augmented counterparts are projected into a discriminative space, whereby semantic reciprocal learning, together with our bit-codes collaboratively optimize the coarse generator and encoder in an end-to-end style, so as to synthesize coarse-grained UCGen images. The bottom panel exhibits the fine-grained branch to establish a hierarchical modulation U-Net (namely, HM-UNet) in diffusion models. The HM-UNet is supported by our fine-grained semantic basis, and can be trained in a unified variational lower bound for diffusion models. Moreover, at the \textit{Inference Phase}, the model achieves controllable synthesis from purely randomly sampled noise and labels. Coarse images are generated by sampling bit-codes, while fine-grained generation is performed by modulating the HM-UNet with sampled fine-grained codes $s$ and bit-codes.}}
    \label{fig3_pipeline}
    \vspace{-1.5em}
\end{figure*}

\section{Related Works}
\label{sec:ref}
\subsection{Conditional Generation without Label Priors}
UCGen aims to enable controllable generative modeling without the reliance on manually annotated labels, by uncovering intrinsic structures within data. A straightforward way of achieving UCGen is to separate the latent representation as independent semantic factors or clusters when training DGMs \cite{InfoGAN, mukherjee2019clustergan, liu2020diverse, GMMGAN}. The so established semantic factors and clusters typically correspond to various semantic attributes, ensured by either maximizing the mutual information between latent codes and generated samples \cite{InfoGAN} or by enforcing clustering constraints \cite{mukherjee2019clustergan, liu2020diverse, GMMGAN}. However, imposing predefined clustering assumptions on complex data manifolds is inherently challenging and often results in training instability or posterior collapse \cite{locatello2019challenging, wang2021posterior}. Moreover, given the range of standard Gaussian noise, the numbers of clusters are also restricted, which further prevents the model from representing the rich and fine-grained semantics present in large-scale, high-resolution datasets. 

On the other hand, there are also strategies that achieve UCGen by two stages, in which pseudo-labels are first obtained via self-supervised learning frameworks, such as SimCLR \cite{chen2020simple} or DINO \cite{caron2021emerging}. Upon those pseudo-labels, existing cDGMs were employed as training backbones for conditional generation \cite{hu2023self, casanova2021instance}. Although these representations capture high-level semantic information, a fundamental mismatch persists between the deterministic pseudo-labels and the stochastic noise that alters the generation process. As a result, the generative noise is weakly coupled with the conditioning semantics, leading to ambiguity, particularly in low-density regions of the data manifold, where random variations cannot be reliably aligned with the intended semantic condition. This misalignment often leads to uncontrolled appearance variations or semantic inconsistency in the generated samples. In parallel, several recent studies have explored the use of self-supervised signals to improve unconditional generation quality. Notably, methods such as REPresentation Alignment (REPA) \cite{yu2024representation} aim to enhance convergence and fidelity by aligning noisy hidden states with clean representation targets. While these approaches also leverage semantic representations, they primarily focus on image quality rather than the conditional control, thus lacking explicit mechanisms for semantic modulation when achieving UCGen. As illustrated in Fig.~\ref{fig2_paradigm}, distinct from the paradigms of learning clusters or using pseudo-labels,  we explicitly construct a structured latent space with binary codes through a reciprocal learning strategy, which can effectively encode real-world semantics. By enforcing an invertible mapping between generation and encoding, the proposed framework establishes a strong semantic correspondence between latent variables and generated content, effectively resolving the irrelevance issue observed in pseudo-labeling-based conditioning methods.

\subsection{Latent Space Structuring and Modulation}
Another way capable of achieving UCGen is to post-process the latent space given pre-trained DGMs, by either post-analysis or discrete quantization techniques. On the one hand, the post-analysis methods, such as GANSpace \cite{harkonen2020ganspace} and SeFa \cite{shen2021closed}, perform semantic exploration on the latent spaces of pre-trained GANs or diffusion models to identify interpretable directions. While these approaches can uncover meaningful semantic variations, they remain passive discovery techniques constrained by the fixed and often entangled representations learned from pre-trained models. In addition, they predominantly rely on linear operations, such as principal component analysis, implicitly assuming that semantic attributes correspond to orthogonal directions. This assumption is insufficient for modeling the complex, nonlinear feature hierarchies required for fine-grained control, particularly in high-dimensional latent spaces \cite{tzelepis2021warpedganspace, kwon2023diffusion}.

On the other hand, the discrete latent representations, including VQ-VAE \cite{van2017neural} and VQ-GAN \cite{esser2021taming},  quantize the latent space using a finite codebook. Although discretization introduces a strong structural prior and improves representation compactness, this operation fundamentally violates the independent and identically distributed (\textit{i.i.d.}) assumption of the sampling noise. Consequently, these models typically require complex autoregressive or transformer-based decoders for generation, which complicates sampling and prevents smooth interpolation in the latent space. However, we propose a unified framework that integrates structured latent representations with the expressive power of diffusion models. Unlike VQ-based discretization, the structured latent space preserves the \textit{i.i.d.} noise property, thereby maintaining a continuous generation manifold that enables stable sampling and smooth semantic interpolation. The overall process is optimized in an end-to-end manner for UCGen, instead of post-analysis on fixed generating spaces. More importantly, 
by introducing a hierarchical modulation mechanism within the diffusion process, our CoFi-UCGen network performs layer-wise modulation using coarse global bit-codes and fine-grained local features, thus to the best of our knowledge, achieving the first coarse-to-fine control capability that is not attainable within existing global conditioning strategies.

\section{Methodology}
We introduce the proposed CoFi-UCGen in detail in this section, which is able to achieve UCGen without any label priors in both training and inference stages. The overall framework of our CoFi-UCGen is illustrated in Fig. \ref{fig3_pipeline}, in which the semantic reciprocal learning in Section \ref{method:recip} acts as the prerequisite to ensure semantic consistency when training our proposed CoFi-UCGen. We then introduce the coarse-grained framework in Section \ref{method:coarse_grained}, followed by the fine-grained UCGen in Section \ref{method:fine_grained}. Please note that during the inference stage, we are able to generate conditional images, by purely sampling from random noise and labels to achieve both coarse- and fine-grained UCGen, without the need of input real-world images.

\subsection{Semantic Reciprocal Learning}
\label{method:recip}
The real-world high-dimensional images essentially reside on low-dimensional semantic manifold, on which the latent spaces $\mathcal{Z}$ of DGMs naturally possess high-level semantics \cite{kingma2013auto,arjovsky2017towards,doersch2016tutorial,rombach2022high}. Such latent spaces therefore provide a structured and semantically meaningful representation for real-world images. The majority of DGMs inherently operate in a forward-only manner, mapping latent codes to images without explicitly constraining the inverse mapping. However, achieving UCGen essentially requires to  represent semantics within images, the prerequisite to the follow-up optimization towards the distinct semantics. Therefore, the latent generating space of DGMs provides a well-defined foundation on the richness of semantics, once we are able to invert back to the latent space. Thus, to fully exploit the semantic structure of $\mathcal{Z}$ for downstream coarse-to-fine conditional generation, we introduce the \textit{semantic reciprocal learning}, a general self-supervised framework that enforces cycle semantic consistency between the latent space and the image domain.

We may need to point out that different from existing inversion of DGMs \cite{xia2022gan, roich2022pivotal, mokady2023null} that typically operates upon fixed (or slightly fine-tuned) generators, our reciprocal learning optimises in conjunction between the encoder and the generator, thus capable of improving the latent space for semantic completeness. In contrast, existing DGM inversion can only locate the semantics, without the ability of altering the latent space. More importantly, instead of regularising the similarity between images, we essentially aim to reconstruct the semantics within the latent spaces. In other words, the generator $\mathcal{G}$ in our reciprocal learning framework first synthesizes images from random latent codes, and then the encoder $\mathcal{E}$ maps images back to the same latent codes, by reconstructing semantic representations. Therefore, our semantic reciprocal learning does not require these two components to be symmetric or invertible; instead, it enforces semantic consistency under composition. As shall be introduced shortly in Sections \ref{method:coarse_grained} and \ref{method:fine_grained}, our semantic reciprocal learning acts as the cores for both coarse and fine-grained UCGen when encoding the latent spaces in both GANs and diffusion models.

More specifically, given a latent code $\mathbf{z}$ sampled from the structured latent space $\mathcal{Z}$, the generator produces an image $\mathcal{G}(\mathbf{z})$, which is then re-encoded by $\mathcal{E}$. Since the latent space $\mathcal{Z}$ inherently captures meaningful semantic structure, discrepancies measured in $\mathcal{Z}$ provide a more informative signal than pixel-wise reconstruction errors in the image domain. In particular, enforcing consistency in latent space directly constrains the generator to preserve semantic attributes encoded in $\mathbf{z}$. Accordingly, we define the semantic reciprocal learning objective as the $l_2$ distance between the original latent code and its re-encoded counterpart, given by
\begin{equation}
r_{z} = | \mathcal{E}(\mathcal{G}(\mathbf{z})) - \mathbf{z} |_2^2, ~\mathbf{z}\in\mathcal{Z}
\label{eq:recip_loss}
\end{equation}
which encourages semantic alignment between generation and encoding whilst avoiding explicit reconstruction constraints in the image space.

Please note that $\mathbf{z}\in\mathcal{Z}$ is randomly sampled from the latent noise space, which allows the reciprocal loss $r_z$ to be optimized without requiring paired real-world images. By minimizing \eqref{eq:recip_loss}, the encoder $\mathcal{E}$ is encouraged to produce latent representations that are consistent with the structured latent space $\mathcal{Z}$ when applied to generate samples. This constraint regularizes $\mathcal{E}$ to align its output distribution with the latent codes employed by the generator, thereby facilitating stable semantic correspondence between encoding and generation.
\begin{remark}
When $\mathcal{G}$ generates images that lie close to the support of real data, the reciprocal learning objective implicitly enables $\mathcal{E}$ to extract fine-grained semantic representations from real-world images, despite the reciprocal loss being optimized using latent codes $\mathbf{z}$. This effect arises from distributional alignment rather than explicit reconstruction or invertibility assumptions.
\end{remark}

We further provide in Lemma \ref{lemma:info}  that our reciprocal learning essentially maximises the mutual information between the latent features and generated images, thus promoting the controllability for our unsupervised conditional generation. We may also need to point out that InfoGAN \cite{InfoGAN} also proposes to maximise the mutual information by variational networks. However, in contrast to InfoGAN, we achieve this by implicitly maximizing $ I(\mathbf{z}; \mathcal{G}(\mathbf{z}))$ through the reciprocal difference $r_z$, without the need for additional variational networks. Indeed, our motivation by obtaining representations is fundamentally different from InfoGAN that relies on coarse-grained pseudo-labels, in which the proposed reciprocal difference coincides with maximising the mutual information. 

\newtheorem{lemma}{Lemma}
\begin{lemma}
\label{lemma:info}
Assume the encoder $\mathcal{E}$ performs a noisy estimation of the latent codes $\mathbf{z}$ by:
    \begin{equation}
    q_{\phi}(\mathbf{z} | \mathcal{G}(\mathbf{z})) = \mathcal{N}(\mathbf{z};\mathcal{E}(\mathcal{G}(\mathbf{z})), \sigma^2\mathbf{I})
    \label{estimation_enc}
    \end{equation}
where $\phi$ is the parameter of $\mathcal{E}$. Then, the mutual information between the latent features and generated images $I(\mathbf{z}; \mathcal{G}(\mathbf{z}))$ satisfies the following inequality regarding our reciprocal difference $r_z$:
    \begin{equation}
    I(\mathbf{z}; \mathcal{G}(\mathbf{z})) \geq -\frac{1} {2\sigma^2}\| \mathcal{E}(\mathcal{G}(\mathbf{z})) - \mathbf{z} \|_2^2 + \mathrm{const},
    \label{lowerbound_recip}
    \end{equation}
where $\mathrm{const}$ denotes constant irrelevant to $\mathbf{z}$.
\end{lemma}

\begin{proof}
For detailed proof, please refer to Section S-I in the Supplementary Material.
\end{proof}

\subsection{Coarse Semantics by Adversarial Bit-Codes Learning}
\label{method:coarse_grained}
\begin{figure}[t]
    \centering
    \includegraphics[width=\linewidth]{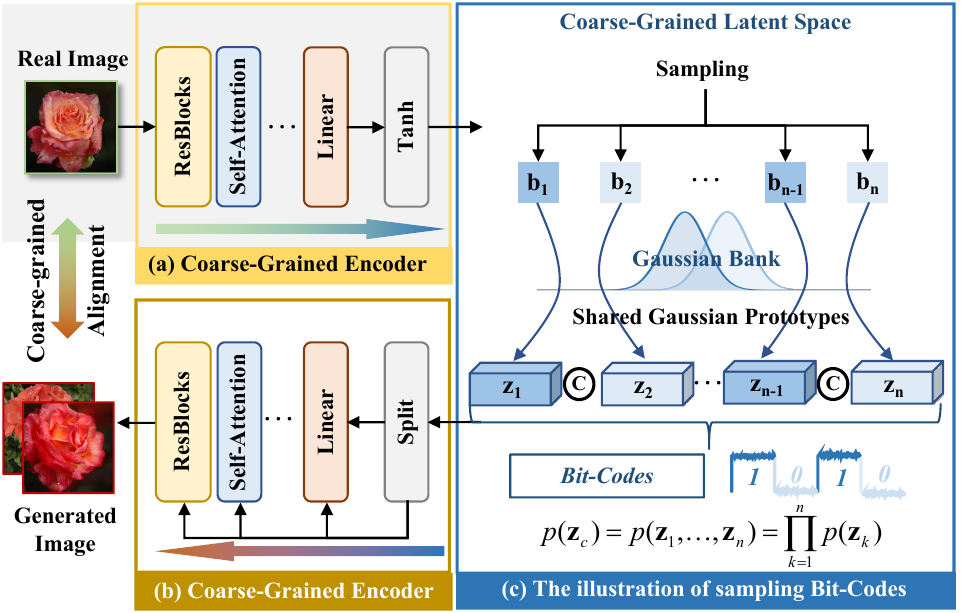}
    \caption{Overview of the training phase of coarse-grained framework in our CoFi-UCGen. Note that real images and the encoder (shaded in gray) are utilized exclusively during training to derive self-supervised semantic conditions, and are discarded during inference. (a) The coarse-grained encoder maps a real image to a compact latent bit-codes through stacked coarse-grained encoder residual blocks and self-attention layers. (b) The coarse-grained generator aggregates the the bit-codes to synthesize UCGen images via corresponding coarse-grained generator residual blocks, enabling coarse-level semantic alignment between real and generated samples. (c) The proposed bit-code ensures consistent and distinct semantics during learning, theoretically proved by \textit{i.i.d} factorized latent segments to disentangle coarse-grained semantics by each bit.}
    \label{fig_framework1}
    \vspace{-1em}
\end{figure}
In this section, we further propose an adversarial semantic reciprocal learning strategy, together with the new format of bit-codes for coarse-grained UCGen; this also serves as the foundation for our hierarchical generation framework. The overall framework for our coarse-grained UCGen is illustrated in Fig.~\ref{fig_framework1}, whereby our methodology is essentially the synergy between a rigid \textit{latent structure} and a robust \textit{learning objective}. First, we introduce the \textit{bit-codes}, a structured probabilistic prior designed to impose explicit semantic disentanglement. By strictly partitioning the latent space into independent segments, we prove that bit-codes provide the necessary structural scaffolding, ensuring that distinct coarse attributes are represented without statistical entanglement. Then, to integrate this prior into the real-world image distribution, we leverage an adversarial reciprocal learning strategy, which achieves complete, consistent and distinct semantics when training DGMs in an end-to-end style. In other words, while the bit-codes alter the topological organization of the semantics, the adversarial learning guides the optimization to enforce a rigorous distributional alignment between the encoded real images and the structured prior. 

Therefore, the proposed semantic reciprocal learning strategy essentially ensures latent space semantic consistency between generation and encoding. Upon this, we further impose an explicit coarse-grained structure when training DGMs, so as to enable the latent space with the capability of automatically representing distinct semantics. This is achieved by our newly proposed \textit{bit-codes}, which are structured into independent segments, with each segment corresponding to a distinct coarse-grained semantic attribute. In other words, the bit-codes aim to encode global semantic attributes to provide a stable and structured representation of coarse-grained semantics, whilst ignoring fine-grained variations. As shall be introduced in Section~\ref{method:fine_grained}, the coarse-grained latent representation acts as the semantic anchor for the hierarchical modulation in diffusion models, in which the disentangled fine-grained semantics are established by a novel semantic basis.

 \begin{figure}[t]
    \centering
    \includegraphics[width=0.95\linewidth]{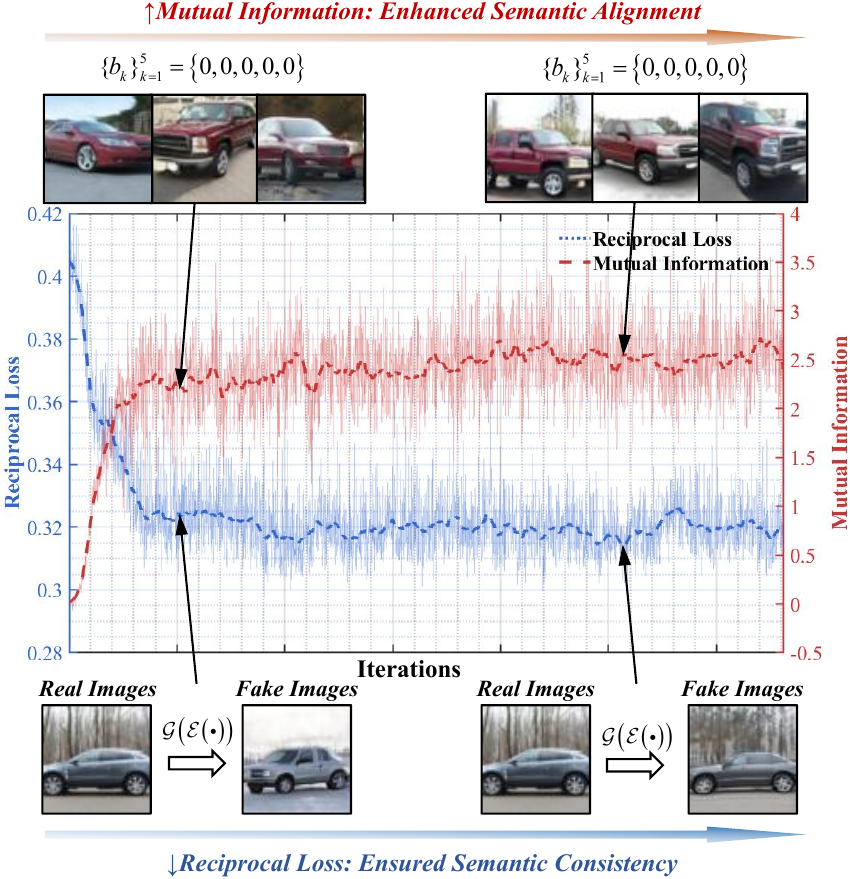}
    \caption{Illustration on the relationship between reciprocal alignment and mutual information. The reciprocal discrepancy $r_z$ (blue) and the estimated mutual information (red) are plotted over training iterations, against Stanford Cars dataset. \textit{Visual snapshots exemplify the physical interpretation:} As the reciprocal loss decreases, the semantic alignment between real inputs and their reconstructions improves (bottom pairs). Conversely, the increase in mutual information correlates with enhanced semantic consistency among samples generated from identical bits $\{b_k\}_{k=1}^{5}$ (top pairs). These empirical observations corroborate our theoretical findings in Lemma~\ref{lemma:info}.}
    \label{fig_recip&mi}
    \vspace{-2em}
\end{figure}

More specifically, let $\mathbf{z}_c \in \mathbb{R}^d$ denote the coarse-grained latent codes produced by the encoder $\mathcal{E}$ in the reciprocal learning framework. We represent $\mathbf{z}_c$ by $n$ segments:
\begin{equation}
    \mathbf{z}_c = [\mathbf{z}_1, \mathbf{z}_2, \dots, \mathbf{z}_k, \dots, \mathbf{z}_n]^\top, 
    \quad \mathbf{z}_k \in \mathbb{R}^{\nicefrac{d}{n}},
    \label{eq:sements}
\end{equation}
where $d$ is assumed to be divisible by $n$. Conceptually, each segment $\mathbf{z}_k$ is designed to encode one coarse semantic factor, and all segments are encouraged to be statistically independent at the prior level. This structured prior is aligned with image-level semantics via adversarial training and reciprocal mapping.

Then, instead of sampling $\mathbf{z}_c$ from either a single Gaussian or a generic Gaussian mixture, we introduce a \emph{bit-code latent prior}. We denote the total number of coarse-grained semantics by $C$, whereas $c \in \{0,\dots,C-1\}$ represents the bit-codes associated with a sample. We encode $c$ into a binary vector $(b_1,\dots,b_n)$ via
\begin{equation}
    b_k = \left\lfloor \frac{c}{2^{k-1}} \right\rfloor \bmod 2, 
    \quad k = 1,2,\dots,n,
    \label{eq:binary_gaussian}
\end{equation}
where $\lfloor \cdot \rfloor$ is the floor operator. The bit $b_k \in \{0,1\}$ acts as a deterministic selector that chooses between two Gaussian modes for the $k$-th segment:
\begin{equation}
    \mathbf{z}_k \mid b_k \sim \mathcal{N}\big( (-1)^{b_k}\bm{\mu}, \,\sigma_z^2\mathbf{I}\big),
    \label{eq:segment_gaussian}
\end{equation}
where $\bm{\mu} \in \mathbb{R}^{\nicefrac{d}{n}}$ controls the separation between the two semantic modes associated with $b_k=0$ and $b_k=1$, and $\sigma_z^2$ controls the intra-mode variability. In this way, the discrete variables $\{b_k\}$ select coarse semantic modes, while the Gaussian noise within each segment preserves continuous and fine-grained variations. This hybrid construction matches our coarse-to-fine objective, namely, bit-codes structure for coarse controllability, and continuous variability for the follow-up fine-grained synthesis.

Under the above construction, the joint distribution of the latent bit-codes $\mathbf{z}_c$ admits a factorized form. In particular, when the pseudo-codes $c$ are sampled uniformly from the binary code space, the resulting segments $\{\mathbf{z}_k\}$ become \textit{i.i.d}, which is ensured by the following Lemma~\ref{lemma: attribute_ortho}. 


\begin{lemma}
    \label{lemma: attribute_ortho}
    Suppose $C = 2^{n}$ and pseudo-codes $c$ are drawn uniformly from $\{0,\dots,C-1\}$. Given $\{\mathbf{z}_k\}_{k=1}^n$ in \eqref{eq:sements} from $(b_1,\dots,b_n)$ defined in \eqref{eq:segment_gaussian}, the segments $\{\mathbf{z}_k\}_{k=1}^n$ are i.i.d random vectors and satisfy
    \begin{equation}
        p(\mathbf{z}_c) = p(\mathbf{z}_1,\dots,\mathbf{z}_n) 
        = \prod_{k=1}^{n} p(\mathbf{z}_k),
        \label{eq:lemma2}
    \end{equation}
    where each $\mathbf{z}_k$ follows the mixture distribution induced by  \eqref{eq:segment_gaussian}.
\end{lemma}

\begin{proof}
For detailed proofs, please refer to Section S-II in the Supplementary Material.
\end{proof}

More importantly, as shall be provided by Theorem \ref{thm:structured_cov}, our bit-codes design can in principle encode an independent semantic factor during training, and thus provides a desiring latent space structure to align with image semantics.

\newtheorem{theorem}{Theorem}
\begin{theorem}
\label{thm:structured_cov}
Suppose $C = 2^{n}$ and pseudo-codes $c$ are drawn uniformly from $\{0,\dots,C-1\}$. Given $\mathbf{z}_c=\{\mathbf{z}_k\}_{k=1}^n$ in \eqref{eq:sements} from $(b_1,\dots,b_n)$ defined in \eqref{eq:segment_gaussian}, the covariance matrix of $\mathbf{z}_c$ admits an orthonormal basis of $n$ leading eigenvectors $\{\mathbf{v}_k\}_{k=1}^n$, each of which exhibits a segment-wise localized structure:
\begin{equation}
    \mathbf{v}_k = [\mathbf{0}^\top, \dots, \underbrace{\hat{\bm{\mu}}^\top}_{\text{Corresponding to }\mathbf{z}_k}, \dots, \mathbf{0}^\top]^\top,
    \label{eq:disentangled_eigenvectors}
\end{equation}
where $\hat{\bm{\mu}}=\nicefrac{\bm{\mu}}{||\bm{\mu}||_2}$ given \eqref{eq:segment_gaussian}; $\mathbf{v}_k$ is an almost zero-valued vector, except for those dimensions that correspond to $\mathbf{z}_k$.

\begin{proof}
For detailed proofs, please refer to Section S-III in the Supplementary Material.
\end{proof}
\end{theorem}

\noindent{\textbf{Remark:} \textit{Given the continuous latent generating space, its eigenspace has been verified to possess rich  semantics \cite{harkonen2020ganspace, kwon2023diffusion}. By our bit-code design, the leading eigenvectors are locally confined within segments, in which modulating the generate process using specific segments $\mathbf{z}_k$ can effectively manipulate orthogonal semantic directions in the data manifold, thus avoiding the attribute leakage.}}

Based on the proposed bit-codes, we are able to use the semantic reciprocal learning strategy as introduced in Section \ref{method:recip}, such that the semantics of real-world images can be represented with distinct attributes. As illustrated in Fig. \ref{fig_recip&mi}, we can clearly observe that with the optimisation on reciprocal learning, our bit-codes can better control the generated semantics, the nature outcome of the maximised mutual information as also proved by Lemma \ref{lemma:info}. In other words, each segment $\mathbf{z}_k$ can encode a distinct semantic attribute without being statistically entangled with the others. Under mild conditions, the covariance of $\mathbf{z}_c$ is block-diagonal with identical blocks, and that the top $n$ eigenvectors of the covariance matrix align with the mean-shift directions induced by $\{b_k\}$. These eigenvectors provide $n$ disentangled linear directions in the latent space, which our semantic reciprocal learning can effectively exploit to align coarse latent factors with image-level semantics. In contrast, existing methods that rely on GMMs \cite{GMMGAN, ying2021unsupervised} essentially cluster all the latent dimensions, such that semantic attributes are encoded in a coupled style. This leads to altering one semantic direction may impact the other semantics. However, our bit-codes force each binary selector $b_k$ to control only the corresponding segment $\mathbf{z}_k$, which ensures a one-to-one correspondence between binary codes and segment-wise mean shifts, providing interpretable and robust control on coarse-grained semantics.

After establishing the distinct semantics within the latent space, the remaining task is to optimize DGMs that are capable of generating images from our bit-codes, ensuring the richness and completeness of semantics. Supported by our semantic reciprocal learning mechanism, we propose to optimize the generation by minimising the distribution discrepancy of semantics between the generated $\mathbf{\widetilde{x}}\in\bm{\mathcal{\widetilde{X}}}$ and real-world images $\mathbf{x}\in\bm{\mathcal{X}}$, i.e., the distribution discrepancy between $\mathcal{E}(\widetilde{\mathbf{x}})$ and $\mathcal{E}(\mathbf{x})$. This essentially aligns the structured latent space $\mathbf{z}_c$ with the image distribution in a principled and stable manner, whereby we adopt the characteristic function (CF) based adversarial training framework \cite{li2022reciprocal}, thus constituting the adversarial semantic reciprocal learning. 

{More specifically, besides the semantic reciprocal learning, we further optimize the distribution discrepancy between $\mathcal{E}(\widetilde{\mathbf{x}})$ and $\mathcal{E}(\mathbf{x})$. Instead of matching distributions in pixel space, we match the distributions of encoder features via a characteristic function discrepancy, which has been proved to provide a stable and expressive measure of distributional mismatch. More specifically, given the image distribution $\bm{\mathcal{X}}$, we are able to calculate the CF $\Phi_{\bm{\mathcal{X}}}(\mathbf{w})$, where $\mathbf{w}$ is the new variables related to CF. A desiring property is that the CF $\Phi_{\bm{\mathcal{X}}}(\mathbf{w})$ possesses one-to-one correspondence to the probabilistic distribution. Therefore, we are able to minimise the upper bound of CF difference, by adversarially finding the ``worst'' $\mathbf{w}$ between $\Phi_{\bm{\mathcal{X}}}(\mathbf{w})$ and $\Phi_{\bm{\mathcal{\widetilde{X}}}}(\mathbf{w})$, given by:
\begin{equation}
\min_{\mathcal{G}}\ \max_{\mathcal{E},\,\mathbf{w}} \ 
C_{\mathcal{W}}\big(\mathcal{E}(\mathbf{x}),\, \mathcal{E}(\widetilde{\mathbf{x}})\big).
\label{eq:minmax_cf}
\end{equation}
where $\mathbf{\widetilde{\mathbf{x}}}=\mathcal{G}(\mathbf{z}_c)$ and $C_{\mathcal{W}}(\cdot,\cdot)$ is the CF discrepancy:
\begin{equation}
	\begin{aligned}
		C_{\mathcal{W}}&(\mathbf{x}, \widetilde{\mathbf{x}}) =\\
		&\int_{\mathbf{w}} \!\!\left( (\Phi_{\mathbf{x}}(\mathbf{w}) - \Phi_{\widetilde{\mathbf{x}}}(\mathbf{w}))(\Phi_{\mathbf{x}}^*(\mathbf{w}) - \Phi_{\widetilde{\mathbf{x}}}^*(\mathbf{w})) \right)^{\frac{1}{2}} \!dF_{\mathcal{W}}(\mathbf{w}),
	\end{aligned}
	\label{eq:CF}
\end{equation}
In this way, by minimising the upper bound $\mathcal{L}_{\mathcal{E}}$, we are able to effectively minimise the CF discrepancy, thus optimizing DGMs to generate images that match the real-world image distribution. }

To further improve the distinction between semantics for UCGen, we improve existing SimCLR~\cite{chen2020simple}, by adopting an augmentation-based contrastive learning strategy to learn invariant and discriminative representations from unlabelled images. Unlike SimCLR that employs a dedicated projection network for contrastive learning and a separate encoder for downstream tasks, our encoder $\mathcal{E}$ is jointly optimized for both targets, namely, the adversarial semantic reciprocal learning and the feature extractor for contrastive learning. 

Concretely, for each real image $\mathbf{x}_i$, we sample two independent augmentations
\begin{equation}
        \hat{\mathbf{x}}_i,\ \hat{\mathbf{x}}_i^+ \sim \mathcal{T}(\mathbf{x}_i),
\end{equation}
where $\mathcal{T}$ denotes a standard stochastic augmentation pipeline (e.g., random cropping and colour jittering). The augmented views are encoded by $\mathcal{E}$ and passed through a lightweight projector $\mathcal{P}$ to obtain contrastive features:
\begin{equation}
    \mathbf{y}_i = \mathcal{P}\big(\mathcal{E}(\hat{\mathbf{x}}_i)\big), 
    \quad 
    \mathbf{y}_i^+ = \mathcal{P}\big(\mathcal{E}(\hat{\mathbf{x}}_i^+)\big).
\end{equation}
Upon this, we regard $(\mathbf{y}_i, \mathbf{y}_i^+)$ as a positive pair and
$(\mathbf{y}_i, \mathbf{y}_j)$ where $j\neq i$ as negative pairs. The contrastive loss is then defined as
\begin{equation}
    \mathcal{L}_{\text{aug}} 
    = -\sum_{i} 
    \log 
    \frac{
        \exp\big( \mathrm{sim}(\mathbf{y}_i, \mathbf{y}_i^+) / \tau \big)
    }{
        \sum_{j: j\neq i} \exp\big( \mathrm{sim}(\mathbf{y}_i, \mathbf{y}_j) / \tau \big)
    },
    \label{eq: contrastive loss}
\end{equation}
where $\mathrm{sim}(\cdot,\cdot)$ denotes a similarity measure (e.g., cosine similarity) and $\tau>0$ is a temperature parameter, which further regularise the encoder to accurately capture the distinct semantics. Therefore, different from existing methods that contrastively augment real-world and generated images \cite{kang2020contragan, jeong2021training}, our contrastive loss operates \emph{exclusively} on real-world images, which, on the one hand, prevents from the generator $\mathcal{G}$ from being contaminated by low-quality augmentation, especially during early training stages. On the other hand, the encoder $\mathcal{E}$ is explicitly regularized to produce invariant and discriminative representations for distinct semantics. As the complementary for our adversarial semantic reciprocal learning, the encoder thus learns latent bit-codes where each segment is encouraged to form a coherent semantic factor in the coarse manner.

\begin{figure*}[t]
    \centering
    \includegraphics[width=0.95\linewidth]{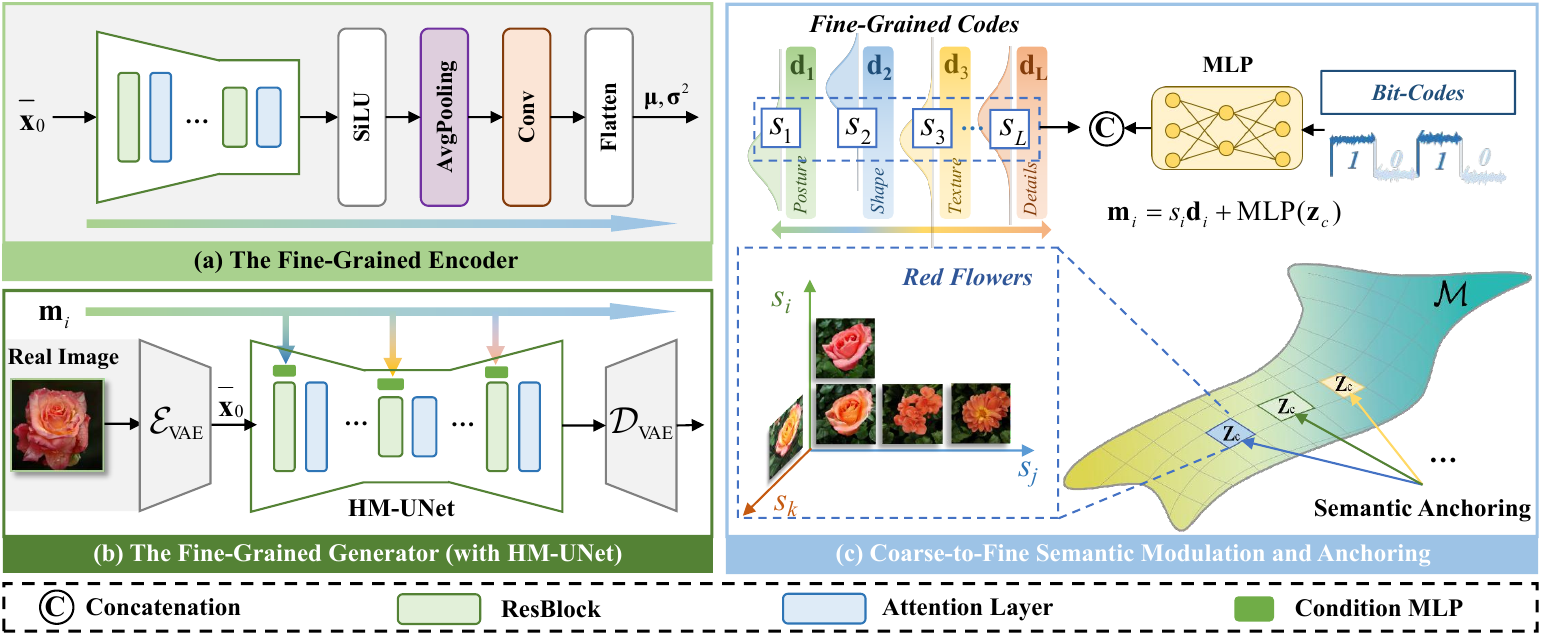}
    \caption{Overview of the training phase of fine-grained framework in our CoFi-UCGen. Note that real images and the encoder (shaded in gray) are utilized exclusively during training to derive self-supervised semantic conditions, and are discarded during inference. (a) The fine-grained semantic encoder extracts fine-grained latent feature ${s}_{1:L}$ by variational Gaussian posterior $\bm{\mu}$ and $\bm{\sigma}^2$, such that the overall optimisation is unified in maximising the ELBO of diffusion models. (b) The hierarchical modulation U-Net (HM-UNet) aggregates the fine-grained latent feature ${s}_{1:L}$ through layer-specific condition multilayer perceptron networks (MLPs), ensuring precise semantic guidance. (c) The hierarchical modulation is established upon the coarse-grained semantics $\mathbf{z}_c$ and then modulates the fine-grained semantics by establishing a semantic basis, which constrains the modulation on a manifold $\mathcal{M}$, along disentangled basis directions to stabilize the generation process.}
    \label{fig_framework2}
    \vspace{-1em}
\end{figure*}

\begin{figure}[t]
    \centering
    \includegraphics[width=.95\linewidth]{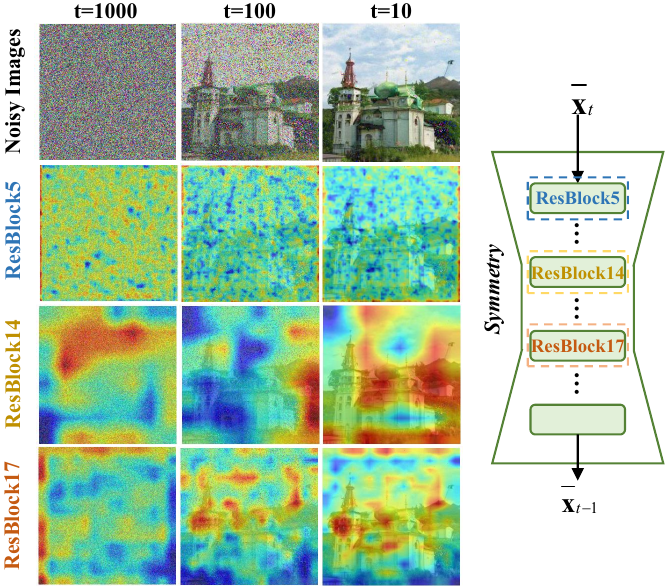}
    \caption{Visualization of the semantic hierarchy in a U-Net-based diffusion model.
Columns show three denoising timesteps ($t=1000, 100, 10$ from left to right), and rows exhibit the noisy inputs and feature heatmaps from 3 representative ResBlocks. Due to the symmetry of the U-Net, we restrict the visualization to the first half of the network. Early ResBlocks (e.g., ResBlock5) mainly respond to local textures and edges, while deeper blocks (ResBlock14, ResBlock17) emphasize object parts and global scene layout. This phenomenon motivates our design of hierarchical modulation, whereby different ResBlocks are conditioned by coarse- and fine-grained semantics.}
    \label{fig_insight}
    \vspace{-1em}
\end{figure}

\subsection{Hierarchical Modulation for Fine-Grained Synthesis}
\label{method:fine_grained}
In our CoFi-UCGen framework, the pre-trained coarse-grained encoder $\mathcal{E}$ establishes a structured and bit-wise disentangled representation $\mathbf{z}_c$, while we propose the fine-grained encoder $\mathcal{E}'$ in this section to control additional latent variables and refine local details. The core principle is that $\mathcal{E}'$ should avoid redundantly re-encoding the coarse semantics already captured by $\mathcal{E}$; instead, it should produce coefficients over a learned \textit{semantic basis} that selectively modulates fine-grained features within the U-Net, which we refer to as HM-UNet. {The detailed architecture and training pipeline of this fine-grained framework are illustrated in Fig.~\ref{fig_framework2}.}

We further propose a new hierarchical modulation upon semantic basis from diffusion models, which achieves the Unprecedented control on the fine-grained semantics without the requirements for any label priors. Our unsupervised fine-grained control is achieved by finding the intrinsic hierarchy of the U-Net architecture, as illustrated in Fig.~\ref{fig_insight}. From this figure, we can conclude that shallow layers primarily encode local textures and noise statistics at early timesteps, while deeper layers progressively capture more global structure and high-level semantic layout at later timesteps. This behaviour motivates a hierarchical strategy that modulates different UNet depths to match semantic granularity for fine-grained control.


More specifically, we denote $\overline{\mathbf{x}}_0$ as the latent representation of real-world images, obtained by a pre-trained variational auto-encoder (VAE) by latent diffusion models \cite{rombach2022high}, whereas $\overline{\mathbf{x}}_{1:T}$ represent the noisy latent features along with the forward diffusion process. Recall that $\mathbf{z}_c = \mathcal{E}(\mathbf{x})$ is the coarse-grained representation. We thus propose the modulation across the residual blocks (ResBlocks) within the diffusion HM-UNet architecture.
We denote the HM-UNet by $L$ ResBlocks $\{\mathrm{ResBlock}_i\}_{i=1}^L$. For each $i$-th block, we propose to learn the semantic basis $\mathbf{d}_i \in \mathbb{R}^d$, where $d$ corresponds to the dimension for the $i$-th ResBlock. The collection $\{\mathbf{d}_i\}_{i=1}^L$ therefore formulates a set of fine-grained semantic directions. Meanwhile, the fine-grained latent vector $\mathbf{s}$, obtained by our fine-grained encoder $\mathcal{E}'$, essentially provides instance-specific coordinates in the semantic basis $\{\mathbf{d}_i\}_{i=1}^L$, thus controlling the fine-grained semantics in our CoFi-UCGen. 

To formulate this, the vector $\mathbf{s}$ is linearly projected to per-block scalars $\{s_i\}_{i=1}^L$ to determine the modulation weights for each ResBlock. The modulation vector $\mathbf{m}_i$ is formulated as:
\begin{equation}
    \mathbf{m}_i = s_i \mathbf{d}_i + \mathrm{MLP}(\mathbf{z}_c),
    \label{eq:modulation_factor}
\end{equation}
where $\mathrm{MLP}(\mathbf{z}_c)$ is a shared coarse-level modulation derived from the coarse code $\mathbf{z}_c$ and broadcasted to all blocks. In other words, $s_i$ reflects the activation strength of the fine-grained semantics along the direction $\mathbf{d}_i$. 

In practice, $\mathbf{m}_i$ is split into scale and shift parameters for the adaptive group normalization (AdaGN) in $\mathrm{ResBlock}_i$:
\begin{equation}
    \begin{aligned}
        \boldsymbol{\gamma_i}, \boldsymbol{\beta_i} &= \mathrm{Split}(\mathbf{m}_i)\\
    \mathrm{AdaGN}(\mathbf{h_i}) &= \boldsymbol{\gamma_i} \odot \frac{\mathbf{h_i} - \boldsymbol{\mu}(\mathbf{h_i})}{\boldsymbol{\sigma}(\mathbf{h_i})} + \boldsymbol{\beta_i},
    \end{aligned}
\end{equation}
where $\mathbf{h_i}$ is the feature map at the $i$-th block; $\boldsymbol{\mu}(\cdot)$ and $\boldsymbol{\sigma}(\cdot)$ are computed per-channel statistics. Therefore, the coarse representation $\mathbf{z}_c$ thus provides a global and block-consistent modulation, while the fine-grained latent feature $\mathbf{s}$ selects and scales semantic basis vectors $\{\mathbf{d}_i\}$ to refine local details in a layer-specific manner.


More importantly, to ensure $\mathbf{s}$ as a valid stochastic latent variable that can be randomly sampled during inference, we impose a variational structure on the encoder $\mathcal{E}'$ when obtaining $\mathbf{s}$. Rather than obtaining $\mathbf{s}$ deterministically, $\mathcal{E}'$ predicts the parameters of an approximate Gaussian posterior
\begin{equation}
    q_{\phi_f}(\mathbf{s}\mid \overline{\mathbf{x}}_0)
    = \mathcal{N}\!\Big(
        \mathbf{s};\,
        \bm{\mu}_{\phi_f}(\overline{\mathbf{x}}_0),\,
        \bm{\sigma}^2_{\phi_f}(\overline{\mathbf{x}}_0)
    \Big),
\end{equation}
where $\phi_f$ denotes the learnable parameters of $\mathcal{E}'$. We then sample $\mathbf{s}$ via the reparametrization trick:
\begin{equation}
    \mathbf{s} = \bm{\mu}_{\phi_f}(\overline{\mathbf{x}}_0)
    + \bm{\sigma}_{\phi_f}(\overline{\mathbf{x}}_0) \odot \boldsymbol{\epsilon}_f,
    \quad \boldsymbol{\epsilon}_f \sim \mathcal{N}(\mathbf{0}, \mathbf{I}),
    \label{eq:repara}
\end{equation}
where $\odot$ denotes element-wise multiplication. During training, this posterior is constrained by a standard normal prior $p(\mathbf{s}) = \mathcal{N}(\mathbf{s}; \mathbf{0}, \mathbf{I})$, allowing the diffusion model $p_{\theta_d}$ to be conditioned on the sampled fine latent $\mathbf{s}$.


Correspondingly, by obtaining $\mathbf{s}$ via \eqref{eq:repara}, we prove in Lemma \ref{lemma:diffusion_ELBO} that the overall optimisation on our fine-grained control is consistent with the training objective of diffusion model, as a novel hierarchical latent variable model. In this way, the joint distribution of the reverse process, conditioned on the coarse code $\mathbf{z}_c$, can be factorizes as:
\begin{equation}
    p_{\theta_d}(\overline{\mathbf{x}}_{0:T}, \mathbf{s} \mid \mathbf{z}_c)
    = p(\mathbf{s}) \, p(\overline{\mathbf{x}}_T) \prod_{t=1}^T p_{\theta_d}(\overline{\mathbf{x}}_{t-1} \mid \overline{\mathbf{x}}_t, \mathbf{z}_c, \mathbf{s}),
    \label{eq:joint_diffusion_s}
\end{equation}
where $\theta_d$ denotes the learnable parameters of diffusion model, $p(\overline{\mathbf{x}}_T) = \mathcal{N}(\overline{\mathbf{x}}_T; \mathbf{0}, \mathbf{I})$ is the standard Gaussian prior for the diffusion model, and $p_{\theta_d}(\overline{\mathbf{x}}_{t-1} \mid \overline{\mathbf{x}}_t, \mathbf{z}_c, \mathbf{s})$ represents the learned denoising transitions.

For forward process, we approximate the intractable posterior using a variational distribution $q$ that factorizes into a fine-grained encoder and a fixed forward diffusion process:
\begin{equation}
    q(\mathbf{s}, \overline{\mathbf{x}}_{1:T} \mid \overline{\mathbf{x}}_0)
    = \underbrace{q_{\phi_f}(\mathbf{s} \mid \overline{\mathbf{x}}_0)}_{\text{Fine-grained Encoder}} \cdot \underbrace{q(\overline{\mathbf{x}}_{1:T} \mid \overline{\mathbf{x}}_0)}_{\text{Forward Process}},
    \label{eq:q_joint}
\end{equation}
where $q(\overline{\mathbf{x}}_{1:T} \mid \overline{\mathbf{x}}_0) = \prod_{t=1}^T q(\overline{\mathbf{x}}_t \mid \overline{\mathbf{x}}_{t-1})$ follows the standard fixed Markov chain in denoising diffusion probabilistic models (DDPMs).

\begin{lemma}
\label{lemma:diffusion_ELBO}
{The evidence lower bound (ELBO) of the conditional log-likelihood $\log p_{\theta_d}(\overline{\mathbf{x}}_0|\mathbf{z}_c)$ consists of a reconstruction term, a diffusion consistency term, and a latent regularization term, which can be formulated as follows:
\begin{equation}
\label{eq:lemma3}
\begin{aligned}
&\log p_{\theta_d}(\overline{\mathbf{x}}_0|\mathbf{z}_c) 
\geq \mathbb{E}_{Q} \Big[\log p_{\theta_d}(\overline{\mathbf{x}}_0 \mid \overline{\mathbf{x}}_1, \mathbf{z}_c, \mathbf{s})\Big] \\
& - \sum_{t=2}^{T} \mathbb{E}_{Q} \Big[ D_{\mathrm{KL}}\big( q(\overline{\mathbf{x}}_{t-1} \mid \overline{\mathbf{x}}_t, \overline{\mathbf{x}}_0) \,\Vert\, p_{\theta_d}(\overline{\mathbf{x}}_{t-1} \mid \overline{\mathbf{x}}_t, \mathbf{z}_c, \mathbf{s}) \big) \Big] \\
& - D_{\mathrm{KL}}\big( q_{\phi_f}(\mathbf{s} \mid \overline{\mathbf{x}}_0) \,\Vert\, p(\mathbf{s}) \big) + \mathrm{Const},
\end{aligned}
\end{equation}
where $Q$ denotes the joint variational posterior $q(\mathbf{s}, \overline{\mathbf{x}}_{1:T}|\overline{\mathbf{x}}_0)$ and $\mathrm{Const}$ collects constant terms independent of learnable parameters.}
\end{lemma}
\begin{proof}[Proof]
For detailed proofs, please refer to Section S-IV in the Supplementary Material.
\end{proof}

Finally, based on the reciprocal learning strategy as introduced in Section \ref{method:recip}, we are able to maximise the ELBO given in Lemma \ref{lemma:diffusion_ELBO} by a composite of objective functions:
\begin{equation}
    \mathcal{L}_{\text{total}} = \mathcal{L}_{\text{diff}} + \lambda_{\text{KL}} \mathcal{L}_{\text{KL}} + \lambda_{\text{recip}} \mathcal{L}_{\text{recip}},
    \label{eq:total_loss}
\end{equation}
where $\lambda_{\text{KL}}$ and $\lambda_{\text{recip}}$ are balancing hyperparameters; $\mathcal{L}_{\text{diff}}$, $\mathcal{L}_{\text{KL}}$ and $\mathcal{L}_{\text{recip}}$ denote the diffusion loss, latent regularisation and reciprocal loss, as shall be elaborated in the following.
\begin{itemize}
    \item \textbf{Diffusion Loss.} The primary term $\mathcal{L}_{\text{diff}}$ corresponds to the standard noise-prediction objective derived from the variational inference:
\begin{equation}
    \mathcal{L}_{\text{diff}} = \mathbb{E}_{t, \overline{\mathbf{x}}_0, \boldsymbol{\epsilon}, \mathbf{s}} \left[ \| \boldsymbol{\epsilon} - \boldsymbol{\epsilon}_{\theta_d}(\overline{\mathbf{x}}_t, t, \mathbf{z}_c, \mathbf{s}) \|^2 \right],
\end{equation}
where $\overline{\mathbf{x}}_t$ denotes the noisy feature at step $t$ and $\boldsymbol{\epsilon} \sim \mathcal{N}(\mathbf{0}, \mathbf{I})$.
\item \textbf{Latent Regularization.} To ensure the valid sampling of fine-grained semantics, $\mathcal{L}_{\text{KL}}$ enforces the variational posterior to align with the prior:
\begin{equation}
    \mathcal{L}_{\text{KL}} = D_{\mathrm{KL}}\big( q_{\phi_f}(\mathbf{s} \mid \overline{\mathbf{x}}_0) \,\Vert\, \mathcal{N}(\mathbf{0}, \mathbf{I}) \big).
\end{equation}
\item \textbf{Reciprocal Loss.}
The reconstruction term in the ELBO is traditionally modelled as a pixel-wise error. However, as discussed in Section\ref{method:recip}, pixel-level constraints are often insufficient for capturing high-level semantics. Therefore, we instantiate the reciprocal objective $r_z$ in \eqref{eq:recip_loss} within our diffusion process to implicitly maximize the mutual information $I(\mathbf{s}; \overline{\mathbf{x}}_0^{\text{pred}})$. 
During training, the diffusion model predicts the clean image $\overline{\mathbf{x}}_0^{\text{pred}}$ from the noisy input $\overline{\mathbf{x}}_t$ via Tweedie's formula:
\begin{equation}
    \overline{\mathbf{x}}_0^{\text{pred}}(\overline{\mathbf{x}}_t, \mathbf{s}) = \frac{\overline{\mathbf{x}}_t - \sqrt{1-\bar{\alpha}_t}\boldsymbol{\epsilon}_{\theta_d}(\overline{\mathbf{x}}_t, t, \mathbf{z}_c, \mathbf{s})}{\sqrt{\bar{\alpha}_t}}.
\end{equation}
where $\{\alpha_{1:T}\}$ denotes noise schedule. We then re-encode this predicted $\overline{\mathbf{x}}_0^{\text{pred}}$ using the fine encoder $\mathcal{E}'$ (sharing weights with the training encoder) and enforce consistency in the semantic latent space:
\begin{equation}
    \mathcal{L}_{\text{recip}} = \| \mathbf{s} - \bm{\mu}_{\phi_f}(\overline{\mathbf{x}}_0^{\text{pred}}) \|_2^2.
\end{equation}
By minimizing $\mathcal{L}_{\text{recip}}$, we ensure that the generated content $\overline{\mathbf{x}}_0^{\text{pred}}$ faithfully preserves the fine-grained semantics $\mathbf{s}$, thereby achieving the semantic alignment and controllability without relying on pixel-space reconstruction.
\end{itemize}

\begin{table*}[htp]
\centering
\caption{Evaluations on both the coarse- and fine-grained UCGen performances, against the Standord Cars \cite{krause20133d} and UTKFace \cite{zhifei2017cvpr} datasets. Prec., Rec. and Align. denote Precision, Recall and DINO-Aligned, respectively. The best and second-best scores are \best{red} and \second{blue}.}
\label{tab:coarse_fine_cars_utkface}
\resizebox{0.9\textwidth}{!}{%
\setlength{\tabcolsep}{5pt} 
\renewcommand{\arraystretch}{1.25} 
\begin{tabular}{l ccccc cccccc}
\toprule
\multirow{3}{*}{Method} & \multicolumn{11}{c}{\textbf{Stanford Cars}} \\
\cmidrule(lr){2-12}
 & \multicolumn{5}{c}{Coarse-Grained} & \multicolumn{6}{c}{Fine-Grained} \\
\cmidrule(lr){2-6} \cmidrule(lr){7-12}
 & Backbone & FID$\downarrow$ & IS$\uparrow$ & Purity$\uparrow$ & NMI$\uparrow$ & Backbone & FID$\downarrow$ & IS$\uparrow$ & Prec.$\uparrow$ & Rec.$\uparrow$ & Align$\downarrow$ \\
\midrule
ClusterGAN \cite{mukherjee2019clustergan} & BigGAN &64.11  &2.32 & - & - & BigGAN &37.13	&2.53	&0.0256	&0.3442	&0.6319\\
Self-Cond GAN \cite{liu2020diverse}       & BigGAN &40.61  &2.32   & - & - & BigGAN &38.58	&2.47	&0.0025	&0.6115	&- \\
IC-GAN \cite{casanova2021instance}        & BigGAN &41.03  &2.33 & - & - & BigGAN &45.78	&2.59	&0.0326	&0.5544	&- \\
MIC-GANs \cite{ying2021unsupervised}      & BigGAN &69.74  &2.17 & - & - & BigGAN &71.2	&2.27	&0.0012	&0.3427	&0.7465 \\
SG-DM+SimCLR \cite{hu2023self}            & DDPM    &10.37
 &2.48 & - & - & LDM  &10.37	&2.87	&0.4161	&0.6986	&0.5150 \\
SG-DM+DINO \cite{hu2023self}  & DDPM   &9.56 &2.45 & - & - & LDM &11.28 &\second{2.96}	&0.3615	&0.5366	&0.4747 \\
SG-DM+MSN \cite{hu2023self}               & DDPM &9.92
 &\second{2.56} & - & - & LDM &10.82	&2.74	&0.4076	&0.6884	&0.4650 \\
\midrule 
Ours (Noise)    & DDPM &\second{8.65} &\best{2.58} & - & - & LDM  &\second{9.91}	&2.89 &\second{0.4252} &\second{0.7543}	&\second{0.4009} \\
Ours (Repar.) & DDPM &\best{8.15} &2.41 & - & - & LDM &\best{5.98}	&\best{3.01}	&\best{0.6111} &\best{0.8330}	&\best{0.3913}\\
\midrule[1pt]
\multirow{3}{*}{Method} & \multicolumn{11}{c}{\textbf{UTKFace}} \\
\cmidrule(lr){2-12}
 & \multicolumn{5}{c}{Coarse-Grained} & \multicolumn{6}{c}{Fine-Grained} \\
\cmidrule(lr){2-6} \cmidrule(lr){7-12}
 & Backbone & FID$\downarrow$ & IS$\uparrow$ & Purity$\uparrow$ & NMI$\uparrow$ & Backbone & FID$\downarrow$ & IS$\uparrow$ & Prec.$\uparrow$ & Rec.$\uparrow$ & Align$\downarrow$ \\
\midrule
ClusterGAN \cite{mukherjee2019clustergan} & BigGAN &25.31  &2.04 & - & - & BigGAN &25.36	&2.54	&0.3329	&0.5257	&0.4630 \\
Self-Cond GAN \cite{liu2020diverse}       & BigGAN &25.69  &2.17 & - & - & BigGAN &49.92	&2.62	&0.1563	&0.5946	&-\\
IC-GAN \cite{casanova2021instance}        & BigGAN &21.79  &2.31 & - & - & BigGAN &32.66	&2.34	&0.1113	&0.6022	&- \\
MIC-GANs \cite{ying2021unsupervised}      & BigGAN &77.23  &1.69 & - & - & BigGAN &69.54	&2.32	&0.0236	&0.4789	&0.6343 \\
SG-DM+SimCLR \cite{hu2023self}            & DDPM   &12.79 &2.53 & - & - & LDM   &7.51	&2.57	&0.4944	&0.7728	&0.5212 \\
SG-DM+DINO \cite{hu2023self}              & DDPM    &12.42 &2.54 & - & - & LDM &7.98 &2.63	&0.4958	&0.7870	&0.4290\\
SG-DM+MSN \cite{hu2023self}               & DDPM    &12.49
 &2.50 & - & - & LDM &8.12	&\second{2.64}	&0.4782	&\second{0.8030}	&0.4863 \\
\midrule 
Ours (Noise)    & DDPM &\second{11.57}  &\second{2.58}  & - & - & LDM &\second{7.13}	&\second{2.64}	&\best{0.5255}	&0.7547	&\second{0.2680} \\
Ours (Repar.) & DDPM &\best{8.94}  &\best{2.63}  & - & - & LDM &\best{6.12}	&\best{2.65}	&\second{0.5011}	&\best{0.8350}	&\best{0.2637} \\
\midrule[1pt]
\bottomrule
\end{tabular}%
}
\vspace{-1em}
\end{table*}

\begin{table*}[htp]
\centering
\caption{Evaluations on both the coarse- and fine-grained UCGen performances, against the CUB200 \cite{wah2011caltech} and Oxford102-Flowers \cite{Nilsback08} datasets. Prec., Rec. and Align. denote Precision, Recall and DINO-Aligned, respectively. The best and second-best scores are \best{red} and \second{blue}.}
\label{tab:coarse_fine_cub_flowers}
\resizebox{0.9\textwidth}{!}{%
\setlength{\tabcolsep}{5pt} 
\renewcommand{\arraystretch}{1.25} 
\begin{tabular}{l ccccc cccccc}
\toprule
\multirow{3}{*}{Method} & \multicolumn{11}{c}{\textbf{CUB200}} \\
\cmidrule(lr){2-12}
 & \multicolumn{5}{c}{Coarse-Grained} & \multicolumn{6}{c}{Fine-Grained} \\
\cmidrule(lr){2-6} \cmidrule(lr){7-12}
 & Backbone & FID$\downarrow$ & IS$\uparrow$ & Purity$\uparrow$ & NMI$\uparrow$ & Backbone & FID$\downarrow$ & IS$\uparrow$ & Prec.$\uparrow$ & Rec.$\uparrow$ & Align$\downarrow$ \\
\midrule
ClusterGAN \cite{mukherjee2019clustergan} & BigGAN &69.68  &4.41  &0.0261  &0.1348 & BigGAN &77.73 &4.51  &0.0336  &0.5416  &0.6796 \\
Self-Cond GAN \cite{liu2020diverse}       & BigGAN &41.59  &4.33  &0.0413  &0.2462 & BigGAN &38.63  &4.98  &0.0018  &0.6509  &- \\
IC-GAN \cite{casanova2021instance}        & BigGAN &87.09  &3.69  &-  &- & BigGAN &84.83  & 5.75 &0.0016  &0.4166  &- \\
MIC-GANs \cite{ying2021unsupervised}      & BigGAN &77.10 &2.10  &0.0153  &0.1178 & BigGAN &96.27 &4.29  &0.0015  &0.3924  &0.8589 \\
SG-DM+SimCLR \cite{hu2023self}            & DDPM     &32.13 &4.38 &0.0174 &0.1063 & LDM  &18.18	&5.78	&0.4350	&0.5032	&0.8164 \\
SG-DM+DINO \cite{hu2023self}              & DDPM    &29.07 &4.47  &\second{0.0508} &\second{0.2515} & LDM &15.81 &\best{5.92}	&0.4927	&0.6236	&0.7059 \\
SG-DM+MSN \cite{hu2023self}               & DDPM   &32.64 &\best{4.53}  &0.0472 &0.2336 & LDM &15.30	&5.85	&0.4975	&0.5304	&0.6706 \\
\midrule 
Ours (Noise)    & DDPM & \second{15.56}	&\second{4.48}
 & \best{0.0517} & \best{0.2516} & LDM &\second{12.24}	&5.71	&\second{0.5206}	&\second{0.7666}	&\second{0.6051} \\
Ours (Repar.) & DDPM & \best{12.16}	&4.45
 & \best{0.0517} & \best{0.2516} & LDM &\best{7.37}	&\second{5.88}	&\best{0.8912}	&\best{0.9212}	&\best{0.5938} \\
\midrule[1pt]
\multirow{3}{*}{Method} & \multicolumn{11}{c}{\textbf{Oxford102-Flowers}} \\
\cmidrule(lr){2-12}
 & \multicolumn{5}{c}{Coarse-Grained} & \multicolumn{6}{c}{Fine-Grained} \\
\cmidrule(lr){2-6} \cmidrule(lr){7-12}
 & Backbone & FID$\downarrow$ & IS$\uparrow$ & Purity$\uparrow$ & NMI$\uparrow$ & Backbone & FID$\downarrow$ & IS$\uparrow$ & Prec.$\uparrow$ & Rec.$\uparrow$ & Align$\downarrow$ \\
\midrule
ClusterGAN \cite{mukherjee2019clustergan} & BigGAN&39.81  &3.50  &0.0902  &0.3606 & BigGAN & 83.19 & 2.91 & 0.0085 & 0.1503 & 0.5163 \\
Self-Cond GAN \cite{liu2020diverse}       & BigGAN &45.88  &2.71  &0.1409  &0.3228 & BigGAN & 41.21 & 3.65 & 0.0025 & 0.3294 & - \\
IC-GAN \cite{casanova2021instance}        & BigGAN &72.48  &2.82  &-  &-  & BigGAN & 95.98 & 2.49 & 0.0032 & 0.1213 & - \\
MIC-GANs \cite{ying2021unsupervised}      & BigGAN  &65.84  &2.34  &0.0564 &0.2587 & BigGAN & 104.89& 2.57 & 0.0017 & 0.2134 & 0.7984 \\
SG-DM+SimCLR \cite{hu2023self}            & DDPM    &33.86 &3.44 &0.0761 &0.1282 & LDM    & 19.11 & 3.71 & 0.5026 & 0.4368 & 0.8044 \\
SG-DM+DINO \cite{hu2023self}              & DDPM    &26.18 &3.57 &\second{0.2294} &\second{0.4177} & LDM & 17.88 & 3.72 & \best{0.5424} & 0.4688 & 0.6457 \\
SG-DM+MSN \cite{hu2023self}               & DDPM    &28.85 &\second{3.58}  &0.2001 &0.3412 & LDM & 18.80 & 3.71 & 0.4957 & 0.4440 & 0.7017 \\
\midrule 
Ours (Noise)    & DDPM & \second{16.17}	&\best{3.81}
&\best{0.3045}	&\best{0.5147} & LDM & \second{15.41} & \best{3.96} & \second{0.5198} & \second{0.6601} & \best{0.5144} \\
Ours (Repar.) & DDPM & \best{14.63}	&\second{3.58}
 &\best{0.3045}	&\best{0.5147} & LDM & \best{12.60} & \second{3.91} & 0.5091 & \best{0.7624} & \second{0.5157} \\
\bottomrule
\end{tabular}%
}
\vspace{-1em}
\end{table*}

\section{Experiments}
\subsection{Experimental Settings}
\noindent\textbf{Datasets:}
We evaluated the proposed CoFi-UCGen network on a diverse set of datasets, covering synthetic, low-resolution and high-resolution real-world scenarios. To demystify the conditional generation performance, we first performed evaluations on synthetic dataset consisting of $100k$ samples, drawn from the mixture of von Mises-Fisher distributions. The number of clusters is 4, i.e., 4 classes for assessing UCGen, in which the samples are arranged on a ring manifold. 

On the other hand, for real-world image scenarios, our evaluation focused on both coarse-grained and fine-grained UCGen, which requires the capability of identifying subtle intra-class variations and capturing consistent semantic structures within each specific domain. Therefore, we conducted experiments on 4 standard datasets widely recognized for their rich and consistent fine-grained details: Stanford Cars \cite{krause20133d}, UTKFace \cite{zhifei2017cvpr}, CUB200 \cite{wah2011caltech}, and Oxford102-Flowers \cite{Nilsback08}. These datasets are also the \textit{de facto} datasets for fine-grained analysis and classifications \cite{wei2021fine, lin2015bilinear}, which are therefore inherently suitable for verifying the efficacy of both coarse-grained and fine-grained UCGen in assessing complex local attributes (e.g., textures of petals, car models, facial attributes, etc.). \textit{Please note that the labels in the above datasets were only employed for the evaluation, whereas they were absent when training all the comparing methods for UCGen.} Regarding the resolutions, at the coarse-grained stage, all images of datasets were downsampled to $64\times64$ resolution; this is essentially sufficient to capture global semantic layouts for coarse-grained tasks, e.g., categorical classification. For the fine-grained stage, images were kept $256\times256$ resolution for Stanford Cars, UTKFace, and CUB200. To further assess scalability, images were resized to $512\times512$ on Oxford102-Flowers.

\noindent\textbf{Baselines:}
To the best of our knowledge, our CoFi-UCGen is the first work to achieve unsupervised coarse-to-fine conditional generation. In contrast, existing methods primarily focus on coarse-grained control. For fair comparisons, we adopted the state-of-the-art UCGen methods by varying the number of clusters $K$ to approximate different semantic granularities. More specifically, we set $K=32$ (i.e., $2^5$) for coarse-grained evaluation to satisfies the power-of-two assumption in Lemma \ref{lemma: attribute_ortho} and capture global semantics, and $K=256$ for fine-grained evaluation to approximate detailed variability, aligning with our continuous latent formulation.

Representative state-of-the-art methods were selected from two dominant paradigms. Latent factorization approaches include ClusterGAN \cite{mukherjee2019clustergan}, Self-Cond GAN \cite{liu2020diverse}, and MIC-GANs \cite{ying2021unsupervised}, which impose clustering constraints within GAN latent spaces. Self-supervised pseudo-labelling approaches include the GAN-based method IC-GAN \cite{casanova2021instance} and the diffusion-based method SG-DM \cite{hu2023self}, which rely on pretrained feature extractors for conditional guidance. Their official implementations were used when available, and all baselines were retrained by their corresponding default settings.

\noindent\textbf{Evaluation Metrics:}
We employed almost all relevant metrics to comprehensively assess image fidelity, diversity, and semantic consistency at both coarse and fine granularities. For the coarse-grained stage, we report Fréchet Inception distance (FID) \cite{heusel2017gans}, Inception score (IS) \cite{salimans2016improved}, clustering Purity \cite{schutze2008introduction}, and normalized mutual information (NMI) \cite{strehl2002cluster}. FID and IS evaluate perceptual quality and diversity for generation, computed using $5,000$ generated samples to match the validation splits of the fine-grained benchmarks. Purity and NMI are widely employed to measure the coarse-grained accuracy for conditional generation, focusing on cluster alignment inferred from the coarse latent space and ground-truth categories.

For the fine-grained evaluation, the classification-related metrics, i.e., Purity and NMI, may not be qualified for evaluating the fine-grained semantics that essentially approach textual granularity. We thus report the precision and recall to disentangle fidelity and diversity, where precision measures the fraction of generated samples lying on the real data manifold and recall measures coverage of the data distribution \cite{kynkaanniemi2019improved}. To explicitly evaluate fine-grained consistency, we introduce the DINO-Aligned metric, which measures semantic consistency among images generated under the same condition. For each fine-grained semantic condition, we computed the average pairwise cosine distance among embeddings of all generated samples. Lower DINO-Aligned values indicate stronger semantic coherence and more precise fine-grained control.

\textbf{Implementation Details:} 
All models were trained using the Adam optimizer with a learning rate of $1e^{-4}$. At the coarse-grained stage, we employed the BigGAN-style architecture with a latent dimensionality of $300$, followed by DDPM-based diffusion backbone with a base channel width of $96$. For fine-grained generation, diffusion was performed in the latent space of a pretrained KL-VAE, following the latent diffusion paradigm to reduce computational cost. The diffusion model employed base channel widths of 192 for $256\times256$ resolution and 256 for $512\times512$ resolution. All experiments were conducted on two NVIDIA RTX 4090 GPUs, and training the full coarse-to-fine framework took approximately one week. We shall also release our code and pre-trained weights upon the acceptance.

\subsection{Comparisons on Synthetic Dataset}
 \begin{figure}[t]
    \centering
    \includegraphics[width=0.95\linewidth]{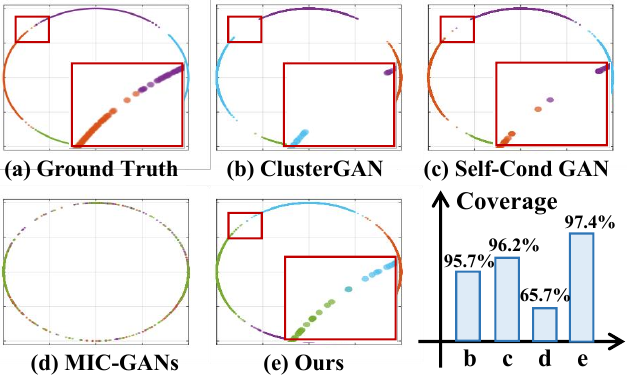}
    \caption{Evaluations on the synthetic dataset, which consists of the mixture of 4 von Mises-Fisher distributions. All generators were trained by 4 fully-connected-layer networks, with hidden size set to 64 and the dimensions of the latent space set to 32. We also employ the Coverage metric as proposed by \cite{naeem2020reliable} to assess the distributional coverage of our model.}
    \label{fig_toydataset}
    \vspace{-1.5em}
\end{figure}
Before evaluating performance on real-world datasets, we first conducted experiments on a synthetic toy dataset to isolate and analyse the behaviour of unsupervised conditional generation under visualisations. The synthetic dataset consisted of $100k$ samples drawn from von Mises-Fisher distributions, forming four clusters with mean directions at $(0, \pi/4, \pi/2, 3\pi/4)$ and a concentration parameter $\kappa=30$, resulting in a 2D ring-shaped manifold.

Since IC-GAN and SG-DM relied on pretrained feature extractors, they were not applicable in this setting and were therefore excluded from the comparison. We instead considered latent factorization-based baselines, including ClusterGAN, Self-Cond GAN, and MIC-GANs, which can be trained from scratch in a fully unsupervised end-to-end manner. Qualitative results for unsupervised conditional generation are reported in Fig.~\ref{fig_toydataset}. From this figure, our CoFi-UCGen method faithfully recovered the underlying cluster structure and generated well-separated conditional samples that closely matched the ground-truth distribution. In contrast, ClusterGAN and Self-Cond GAN tended to produce overly concentrated samples with limited coverage, while MIC-GANs generated intertwined clusters with ambiguous conditional boundaries, suggesting weaker semantic disentanglement.

\begin{figure*}[t]
    \centering
    \includegraphics[width=.95\linewidth]{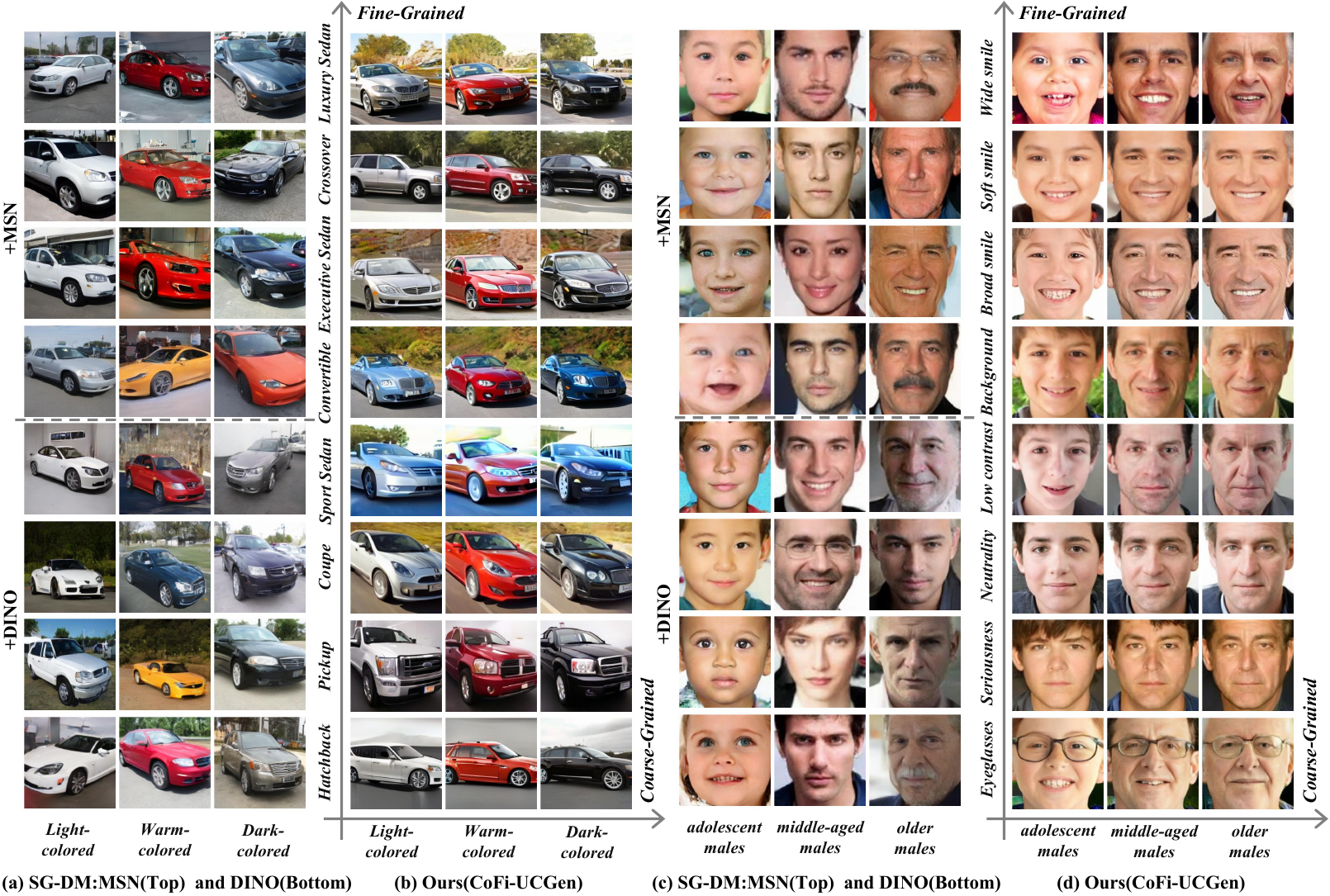}
    \caption{Qualitative comparison of both coarse- and fine-grained UCGen on Stanford Cars (left two panels, $256\times256$ resolution) and UTKFace (right two panels, $256\times256$ resolution). The semantics are best interpreted on two axis, since we do not have access to the ground-truth labels in UCGen. We compare our CoFi-UCGen with the second best method SG-DM, with its two variants, namely, SG-DM+MSN and SG-DM+DINO.}
    \label{fig_cars_utkface}
    \vspace{-1.5em}
\end{figure*}

\begin{figure*}[t]
    \centering
    \includegraphics[width=1\linewidth]{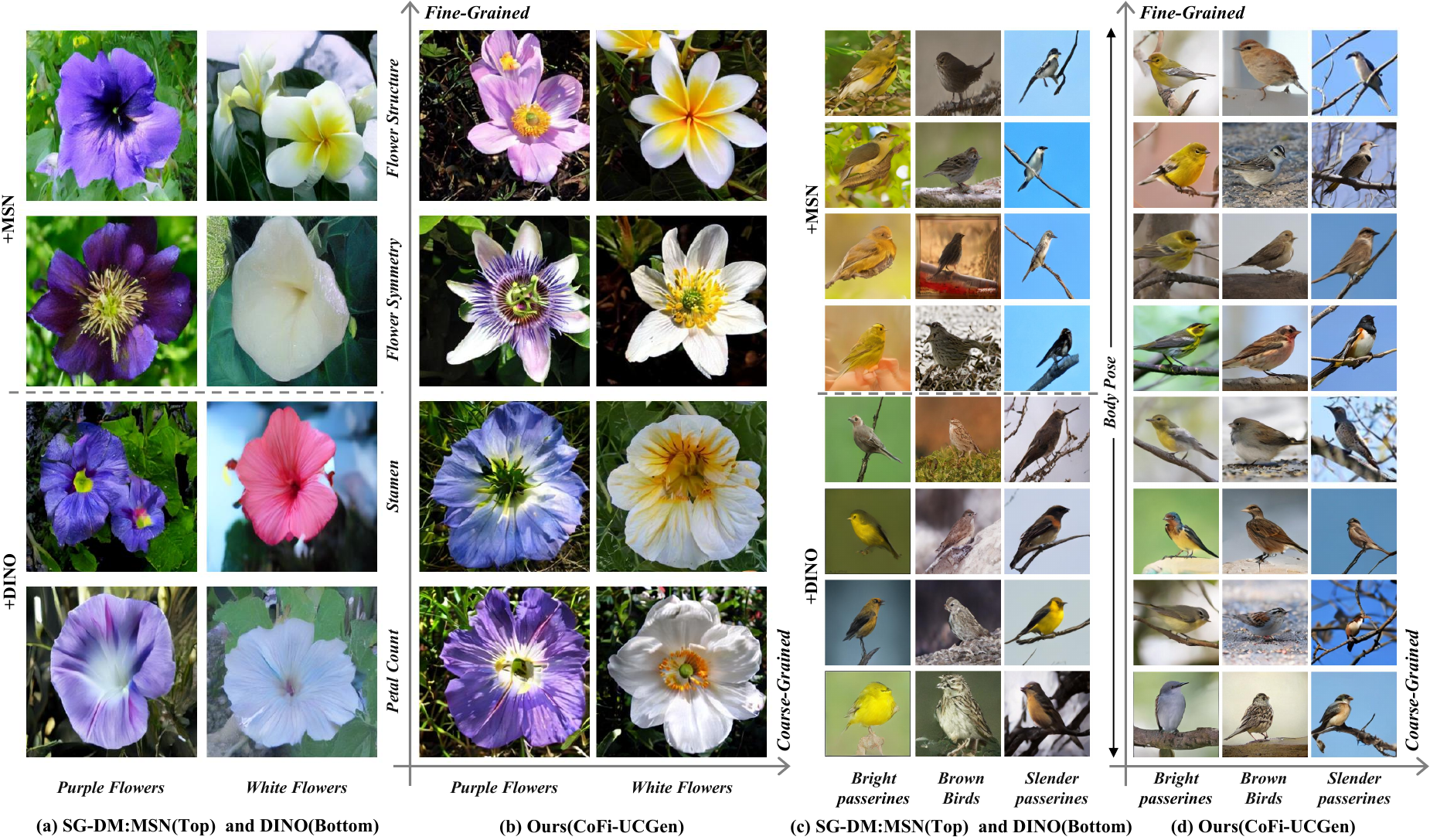}
    \caption{Qualitative comparison of both coarse- and fine-grained UCGen on Oxford102-Flowers (left two panels, $512\times512$ resolution) and CUB200 (right two panels, $256\times256$ resolution). The semantics are best interpreted on two axis, since we do not have access to the ground-truth labels in UCGen. We compare our CoFi-UCGen with the second best method SG-DM, with its two variants, namely, SG-DM+MSN and SG-DM+DINO.}
    \label{fig_flowers_cub200}
    \vspace{-1.5em}
\end{figure*}

\subsection{Quantitative Comparisons}
Table~\ref{tab:coarse_fine_cars_utkface} and Table~\ref{tab:coarse_fine_cub_flowers} report quantitative comparisons of coarse-to-fine image synthesis across the 4 datasets: Stanford Cars, UTKFace, CUB200, and Oxford102-Flowers. We evaluated both the coarse-grained generation stage and the fine-grained diffusion stage using standard metrics that assess perceptual quality, diversity, and semantic consistency.

At the coarse-grained level, our method consistently achieved superior performances across all evaluated datasets. With respect to FID and IS, both variants of the proposed approach outperformed GAN-based baselines as well as diffusion-based UCGen methods, indicating improved global image fidelity and sample diversity. More importantly, on datasets equipped with annotated labels, our model attained the highest Purity and NMI scores, demonstrating that the learned structured latent space effectively captured meaningful global semantics in a fully unsupervised manner. Compared with prior latent factorization methods, such as ClusterGAN and MIC-GANs, the proposed approach yielded more coherent and well-separated latent clusters, whereas diffusion-based UCGen baselines exhibited limited semantic organization at the coarse-grained level.

At the fine-grained stage, again, the proposed hierarchical diffusion framework further enhanced both generative quality and control accuracy. Across all datasets, our method achieved the lowest FID scores, with particularly pronounced improvements on fine-grained benchmarks such as CUB200 and Oxford102-Flowers, indicating high-fidelity synthesis at increased resolutions. In addition, our approach consistently attained higher precision and recall values than competing methods, reflecting superior balance between sample fidelity and distribution coverage. These results suggested that hierarchical modulation effectively alleviated the quality–diversity trade-off commonly observed in prior generative approaches.

Beyond perceptual quality metrics, the proposed method exhibited clear advantages in fine-grained semantic control, as quantified by the proposed DINO-Aligned metric. On all four datasets, our approach achieved the lowest alignment scores, indicating stronger semantic consistency among images generated under identical fine-grained conditions. This improvement was particularly evident on datasets with substantial intra-class variability, such as Stanford Cars and CUB200, highlighting the effectiveness of jointly leveraging a structured coarse latent space and hierarchical fine-grained modulation. Note that Purity and NMI are omitted for the multi-attribute datasets (Stanford Cars, UTKFace) and the instance-guided IC-GAN due to label inapplicability. Additionally, DINO-Aligned is excluded for Self-Cond GAN, as severe mode collapse renders the metric invalid.

Finally, by comparing the last two rows in Tables~\ref{tab:coarse_fine_cars_utkface} and ~\ref{tab:coarse_fine_cub_flowers}, we observe that the reparameterized variant consistently outperformed the noise-based variant across both coarse and fine stages, particularly in terms of FID, recall, and DINO-Aligned scores. This verifies that explicit latent reparameterization facilitates more stable semantic propagation from coarse conditions to fine-grained diffusion, resulting in advanced accurate and controllable image synthesis.

\subsection{Qualitative Comparisons}
Fig.~\ref{fig_cars_utkface} and Fig.~\ref{fig_flowers_cub200} presented qualitative comparisons of coarse-to-fine conditional generation across multiple datasets. In each figure, the horizontal axis corresponded to different coarse-grained conditions, while the vertical axis varied fine-grained conditions within each coarse category. For the best visualisations, we compared the proposed method with two representative competitive baselines that achieved the second-best quantitative performance.

Across all datasets, the proposed CoFi-UCGen network demonstrated a distinction separation between coarse-grained semantics and fine-grained attributes. Along the horizontal direction, the generated images consistently preserved the global semantic category specified by the coarse condition, such as vehicle type, flower colour group, and bird species. Along the vertical direction, fine-grained variations of our CoFi-UCGen were smoothly and consistently controlled, manifesting as changes in attributes such as colour, texture, pose, facial characteristics, and subtle structural details, while maintaining coherence with the corresponding coarse-level semantics. In contrast, the competing methods operated with a single semantic level and therefore did not explicitly distinguish between coarse-grained and fine-grained conditions. As a result, their visualizations varied only along the horizontal axis, without a corresponding vertical dimension for fine-grained control. Although these methods preserved coarse-level semantics to a certain extent, the absence of hierarchical conditioning prevented systematic manipulation of fine-grained attributes.

By further inspecting each row of Fig.~\ref{fig_cars_utkface} and Fig.~\ref{fig_flowers_cub200}, we can also conclude that our CoFi-UCGen method achieved more consistent intra-row semantics and apparent inter-row variations, reflecting stronger fine-grained semantic coherence under fixed coarse conditions. This behaviour was particularly evident on datasets with high intra-class variability, such as Stanford Cars, UTKFace, and CUB200, where subtle fine-grained differences were hard to control without explicit hierarchical modulation. These visual results further verified that the structured coarse latent space served as a stable global semantic anchor, while the hierarchical diffusion mechanism effectively injected fine-grained details without disrupting coarse-level consistency. Overall, the qualitative comparisons demonstrated that the proposed CoFi-UCGen network achieved the best performances on precise, interpretable, and controllable UCGen than existing state-of-the-art baselines.

\begin{table}[t]
    \centering
    \caption{Comparison of representation learning performances on CUB-200 and Flowers datasets, against existing representation learning backbones.The best and second-best scores are \best{red} and \second{blue}.}
    \label{tab:encoder}
    \resizebox{\columnwidth}{!}{%
        \setlength{\tabcolsep}{4pt}
        \renewcommand{\arraystretch}{1.2}
        \begin{tabular}{lc c cc cc}
            \toprule
            \multirow{2}{*}{Method} & \multirow{2}{*}{Backbone} & \multirow{2}{*}{Dim.} & \multicolumn{2}{c}{CUB200} & \multicolumn{2}{c}{Oxford102-Flowers} \\
            \cmidrule(lr){4-5} \cmidrule(lr){6-7}
             & & & Purity$\uparrow$ & NMI$\uparrow$& Purity$\uparrow$& NMI$\uparrow$\\
            \midrule
            
            \multicolumn{7}{l}{\textbf{Supervised Pre-training}} \\
            ResNet-50 &- & 2048 & 0.0457 & 0.2281 & 0.2071 & 0.3296 \\
            DenseNet121 &- &1024 & 0.0460 & 0.2347 & 0.2033 & 0.3261 \\
            VGG19 &- &512 & 0.0417 & 0.2133 & 0.1746 & 0.2815 \\
            
            \midrule
            
            \multicolumn{7}{l}{\textbf{Unsupervised / Self-supervised Pre-training}} \\
            SimCLR \cite{chen2020simple} & ResNet-50 & 2048 & 0.0174 & 0.1063 & 0.0761 & 0.1282 \\
            MAE \cite{he2022masked} & ViT-B/16 & 768 & 0.0320 & 0.1778 & 0.1298 & 0.2583 \\
            DINO \cite{caron2021emerging} & ResNet-50 & 2048 & 0.0411 & 0.2153 & 0.2061 & 0.3204 \\
            DINO \cite{caron2021emerging}& ViT-B/16 & 768 & \second{0.0508} & \second{0.2516} & \second{0.2294} & \second{0.4177} \\
            MSN \cite{assran2022masked}& ViT-B/16 & 768 & 0.0472 & 0.2336 & 0.2001 & 0.3412 \\
            
            \midrule
            
            Ours ($\mathcal{E}$) & ResNet & 300 & \best{0.0517} & \best{0.2516} & \best{0.3045} & \best{0.5147} \\
            \bottomrule
        \end{tabular}%
    }
    \vspace{-1em}
\end{table}

 \begin{figure}[t]
    \centering
    \includegraphics[width=0.95\linewidth]{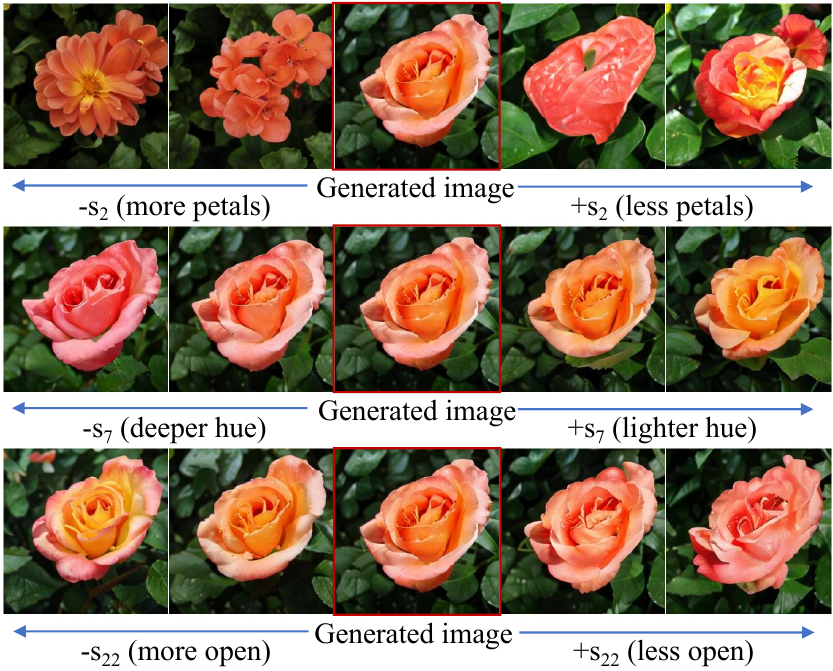}
    \caption{Latent space continuity under hierarchical diffusion modulation. Traversing fine-grained latent factors $\mathbf{s}$ at different diffusion depths induces smooth, semantically meaningful variations in structure, colour, and bloom degree, while preserving coarse-grained semantics (for example, warm-colored flowers in this example), demonstrating stable fine-grained control without altering global semantics.}
    \label{fig_fine_latent}
    \vspace{-1.5em}
\end{figure}

\subsection{Analysis on Learned Representations}
As a ``bonus'', our CoFi-UCGen essentially is able to achieve unsupervised classification, we assessed the representation learning performance of our encoder against a range of state-of-the-art supervised and self-supervised baselines on the CUB200 and Oxford102-Flowers datasets\footnote{We exclude Stanford Cars and UTKFace from this evaluation since they provide multi-attribute annotations rather than a single mutually exclusive class label, which is not directly comparable to standard unsupervised classification protocols.}. As summarized in Table~\ref{tab:encoder}, the comparisons were performed against representative representation learning methods with diverse backbone architectures and embedding dimensionalities. As can be seen from this table, despite operating in a substantially lower-dimensional latent space, our encoder from our CoFi-UCGen consistently achieved the highest Purity and NMI scores on both datasets. Notably, it outperformed not only fully supervised baselines but also large-scale self-supervised models, including ViT-based approaches, by a considerable margin. This performance gain was particularly pronounced on Oxford102-Flowers, indicating that our encoder more effectively captured intrinsic, dataset-specific semantic structures than generic pre-training objectives.

This essentially indicates that jointly learning task-specific representations with a generative objective is more effective than adapting generic and off-the-shelf features for unsupervised semantic discovery. By explicitly optimizing the latent space to facilitate conditional generation, our CoFi-UCGen encoder output discriminative and semantically aligned codes, thereby providing a strong foundation for the subsequent fine-grained diffusion process.

\subsection{Analysis on Learned Latent Space}
To investigate the impact of hierarchical modulation on the semantic variation, we performed linear interpolations along individual dimensions of the fine-grained latent code $\mathbf{s}$, which was injected at different depths of the diffusion mode in our CoFi-UCGen. Fig.~\ref{fig_fine_latent} visualizes the traversal of three representative factors ($s_2, s_7, s_{22}$), revealing a clear relationship between modulation depth and the level of semantic abstraction.

Modulating $s_2$ at early layers induced pronounced structural changes, such as variations in the arrangement and number of petals (Fig.~\ref{fig_fine_latent}, top row), suggesting that shallow layers primarily governed the global geometric layout. In contrast, traversing $s_7$ at intermediate layers mainly affected chromatic attributes (e.g., hue intensity) while preserving the underlying structure. At deeper layers, varying $s_{22}$ resulted in subtle, localized morphological refinements, such as changes in the degree of bloom openness. This progression demonstrated that our framework effectively established an interpretable and semantically ordered latent space, in which the abstraction level of semantic control was orthogonalized by the depth of modulation.

\begin{table}[t]
    \centering
    
    \caption{Ablation study on core components regarding the coarse-grained synthesis in our CoFi-UCGen framework. Please note that the last row depicts the default setting of our CoFi-UCGen network. The best and second-best scores are \best{red} and \second{blue}.}
    \label{tab:ablation_coarse_final}
    \resizebox{\columnwidth}{!}{
        \setlength{\tabcolsep}{6pt} 
        \renewcommand{\arraystretch}{1.2}
        \begin{tabular}{ccc cccc}
            \toprule
            \multicolumn{3}{c}{Components} & \multicolumn{4}{c}{Metrics} \\
            \cmidrule(lr){1-3} \cmidrule(lr){4-7}
            $\mathcal{L}_{\text{recip}}$ & $\mathcal{L}_{\text{aug}}$ & Latent & FID$\downarrow$ & IS$\uparrow$ & Purity$\uparrow$ & NMI$\uparrow$ \\
            \midrule
            \xmark & \cmark & Bit-codes & 76.23 &4.56 &0.0295 &0.1585 \\
            \cmark & \xmark & Bit-codes &51.48 &\best{5.01}	&\second{0.0434}	&\second{0.2281} \\
            \cmark & \cmark & GMM & 95.42 &4.53	&0.0258	&0.1381 \\
            \cmark & \cmark & One-hot &\second{17.52} &4.21	&0.0065	&0.0232 \\
            \midrule
            \cmark & \cmark & Bit-code &\best{15.94}	&\second{4.57} &\best{0.0507}	&\best{0.2516} \\
            \bottomrule
        \end{tabular}
    }

    \vspace{0.4cm} 

    \caption{Ablation study on core components regarding the fine-grained synthesis in our CoFi-UCGen framework. Please note that the last row depicts the default setting of our CoFi-UCGen network.The best and second-best scores are \best{red} and \second{blue}.}
    \label{tab:ablation_fine_final}
    \resizebox{\columnwidth}{!}{
        \setlength{\tabcolsep}{5pt} 
        \renewcommand{\arraystretch}{1.2}
        \begin{tabular}{cccc ccccc}
            \toprule
            \multicolumn{4}{c}{Components} & \multicolumn{5}{c}{Metrics} \\
            \cmidrule(lr){1-4} \cmidrule(lr){5-9}
           $\mathcal{L}_{\text{recip}}$ & $\mathcal{L}_{\text{KL}}$ & Hier. & Enc.& FID$\downarrow$ & IS$\uparrow$ & Prec.$\uparrow$ & Rec.$\uparrow$ & Align$\downarrow$ \\
            \midrule
            \xmark & \cmark & \cmark & $\mathcal{E}$ & 15.58 & 4.98 & 0.3745 & 0.6456 & 0.6133 \\
            \cmark & \xmark & \cmark & $\mathcal{E}$ & 137.21 & 4.06 & 0.0064 & 0.3325 &\best{0.1042}  \\
            \cmark & \cmark & \xmark & $\mathcal{E}$ &25.58  &\second{5.30} &\second{0.5143} &0.6926  &0.7330 \\
            \cmark & \cmark & \cmark & DINO &\second{12.37} &5.19 &0.4333 &\second{0.6977} &0.6138 \\
            \midrule
            \cmark & \cmark & \cmark & $\mathcal{E}$ & \best{12.24} & \best{5.71} & \best{0.5206} & \best{0.7666} & \second{0.6051} \\
            \bottomrule
        \end{tabular}
    }
\end{table}

\subsection{Ablation Study}

\noindent\textbf{Ablations on Coarse-grained UCGen:} Since our CoFi-UCGen is able to achieve both coarse- and fine-grained UCGen, we first ablated the core components regarding the coarse-grained stage, namely, using the semantic reciprocal loss $\mathcal{L}_{\text{recip}}$ to retain semantic consistency, employing the contrastive regularization $\mathcal{L}_{\text{aug}}$ to further improve the distinct semantics, together with the latent space topology (in which we investiagted our bit-codes and we also ablated the other variants). For the variants of latent space topology, we selected the most representative priors employed in end-to-end UCGen systems: the GMM topology \cite{GMMGAN, ying2021unsupervised} and the One-hot topology \cite{mukherjee2019clustergan, InfoGAN}. We report the ablation results in Table~\ref{tab:ablation_coarse_final}. From this table, we can find that removing either $\mathcal{L}_{\text{recip}}$ or $\mathcal{L}_{\text{aug}}$ lead to evident degradation in both generative quality (FID) and semantic separability (Purity/NMI). This verifies that these two terms play complementary roles, namely, $\mathcal{L}_{\text{aug}}$ encourages separability in the latent space, while $\mathcal{L}_{\text{recip}}$ ensures consistency between these latent structures and the synthesized image manifold.

Regarding the latent structure, while alternative formulations can also capture partial semantics, they consistently achieved inferior performances compared with our bit-code design, when building up the latent space of our CoFi-UCGen.
More specifically, the \textit{GMM} architecture exhibited optimization instability. Enforcing separability among numerous clusters forced the Gaussian means to diverge significantly, impairing training robustness and sensitivity to initialization. On the other hand, the \textit{One-hot} design essentially separated the one-hot pseudo-labels from generation noise, thus suffering from the deficiency on control accuracy.  In contrast, our \textit{bit-codes} provided a superior aggregation of discrete and continuous information, enabling stable optimization and achieving superior semantic clustering performance.

\noindent\textbf{Ablations on Fine-grained UCGen:}
For the fine-grained stage, we further ablated the effectiveness of 4 core components. More specifically, the semantic reciprocal loss $\mathcal{L}_{\mathrm{recip}}$ also acts as the core in our fine-grained UCGen, which was further ablated. The proposed HM-UNet (denoted by Hier.) was also verified as one of the most important module to achieve accurate fine-grained control for UCGen. Moreover, we also ablated our KL regularisation $\mathcal{L}_{\mathrm{KL}}$ for latent semantic smoothness, together with our end-to-end learned coarse encoder $\mathcal{E}$ against the pretrained DINO backbone that was dominantly employed in UCGen \cite{caron2021emerging}. The results are listed in Table \ref{tab:ablation_fine_final}. 

From Table \ref{tab:ablation_fine_final}, we can find that without the reciprocal loss, our CoFi-UCGen witnessed degraded performance across all metrics, validating that explicit image and latent alignment is necessary for learning semantically consistent and complete latent representations. Moreover, removing the hierarchical modulation mechanism lead to a noticable drop in perceptual quality, verifying that injecting conditions at multiple network depths is also critical for properly disentangling and rendering multi-scale semantics during generation. Notably, the variant without KL regularization appeared to achieve a better DINO-Aligned score. However, qualitative inspection revealed that this improvement may be misleading. In other words, the latent space became collapsed and discontinuous, causing the model to generate degenerate and noise-like outputs rather than meaningful images, leading to the extremely inferior performance regarding other metrics. The KL term is therefore indispensable for maintaining a valid and smooth sampling manifold that supports coherent generation. 

Finally, replacing our task-oriented coarse encoder with a generic pretrained DINO backbone also resulted in noticeable declines in precision and recall values. This finding reveals that generic recognition features, while effective for discriminative tasks, may be sub-optimal for generation. In contrast, our jointly learned encoder  provided more comprehensive and structurally aligned guidance, enabling the diffusion model to better preserve content fidelity and semantic controllability.

\section{Conclusion}
In this paper, we have addressed the fundamental challenge of unsupervised conditional image generation (UCGen), and introduced a novel coarse-to-fine UCGen (CoFi-UCGen) method that explicitly disentangles global semantics from fine-grained variations, constituting the first effective attempt to support both coarse- and fine-grained conditional generation in a fully unsupervised setting. More specifically, we first introduced an adversarial semantic reciprocal learning theory to enforce semantic consistency and completeness between the image space and the latent space, forming a principled foundation for controllable generation. Leveraging this consistency, we proposed bit-codes to construct a structured coarse-grained latent space, and theoretically established that these codes capture distinct global semantics while still allowing independent noise sampling to preserve diversity. Building on the learned coarse semantics, we further proposed a fine-grained semantic basis and integrated a hierarchical modulation mechanism into diffusion models, enabling layer-wise injection of coarse conditions to progressively steer fine-grained attributes during generation. Extensive experiments verified that without any label priors or pre-trained feature extractors, CoFi-UCGen consistently surpasses existing UCGen methods in image quality, semantic consistency, and control accuracy. These results validate the central premise of our CoFi-UCGen method, namely, explicit coarse-to-fine semantic decomposition is an effective and scalable strategy for addressing the controllability bottleneck in UCGen, thus opening a practical path toward label-free conditional generation with accurate and multi-granular control.

\small
\bibliographystyle{IEEEbib}
\bibliography{ref}
\begin{IEEEbiography}[{\includegraphics[width=1in,height=1.25in,clip,keepaspectratio]{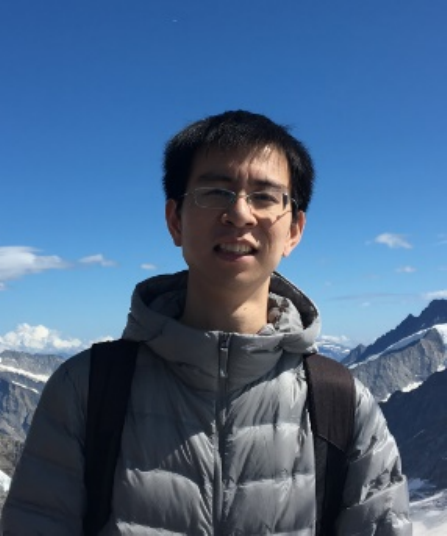}}]{Shengxi Li}
(Member, IEEE) received the Ph.D. degree in electrical and electronic engineering from Imperial College London, London, U.K., in 2021. He is currently a Professor with the School of Electronic Information Engineering, Beihang University, Beijing, China. His research interests mainly include generative models, statistical signal processing, and machine learning. He was the recipient of Young Investigator Award of International Neural Network Society.
\end{IEEEbiography}

\begin{IEEEbiography}[{\includegraphics[width=1in,height=1.25in,clip,keepaspectratio]{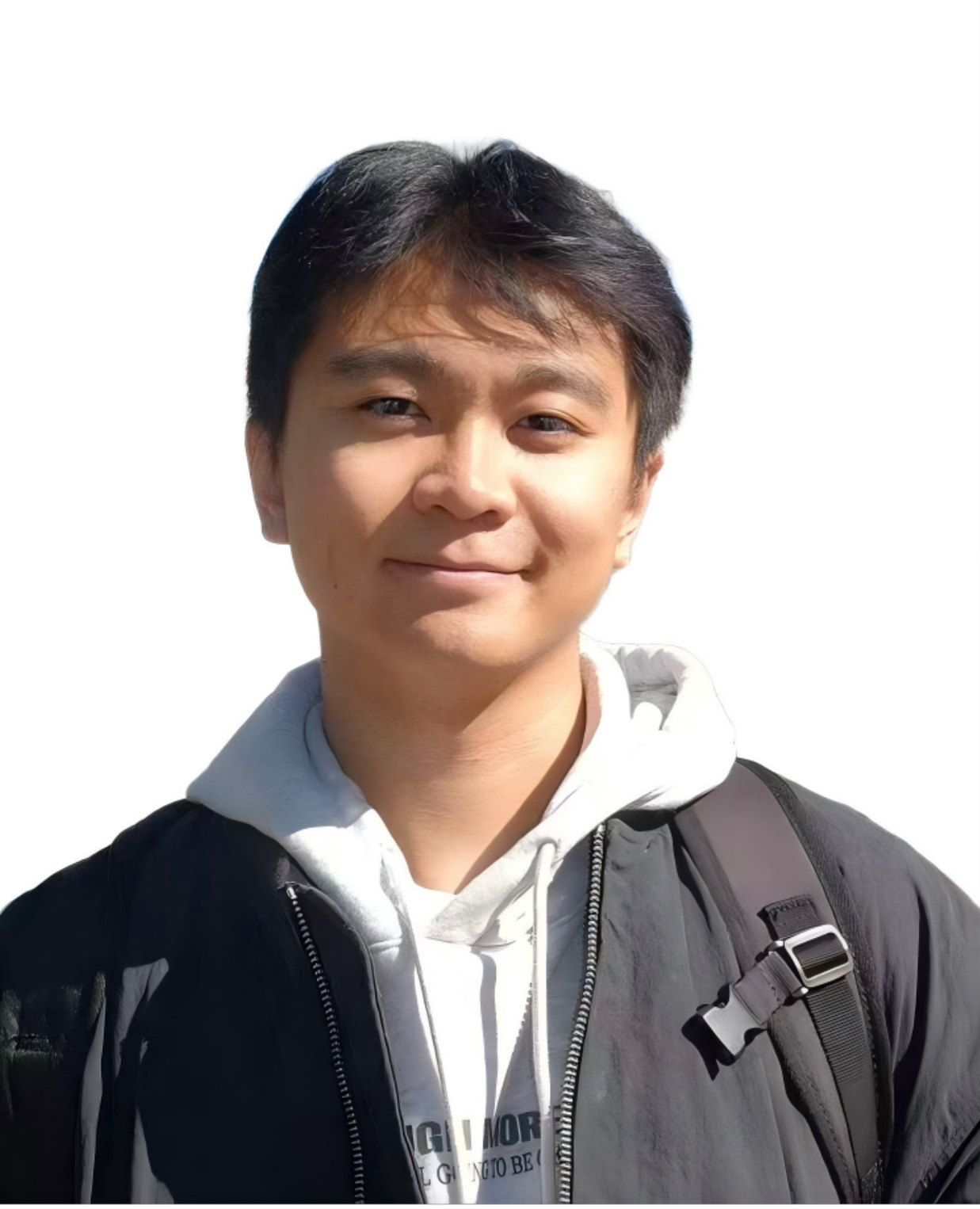}}]{Zhaokun Hu}(Student Member, IEEE) 
received the B.S. degree from the School of Electronic and Information Engineering, Beihang University, Beijing, China, in 2024. He is currently pursuing the MS degree with the School of Electronic and Information Engineering, Beihang University. His research interests mainly include machine learning and generative models.
\end{IEEEbiography}

\begin{IEEEbiography}[{\includegraphics[width=1in,height=1.25in,clip,keepaspectratio]{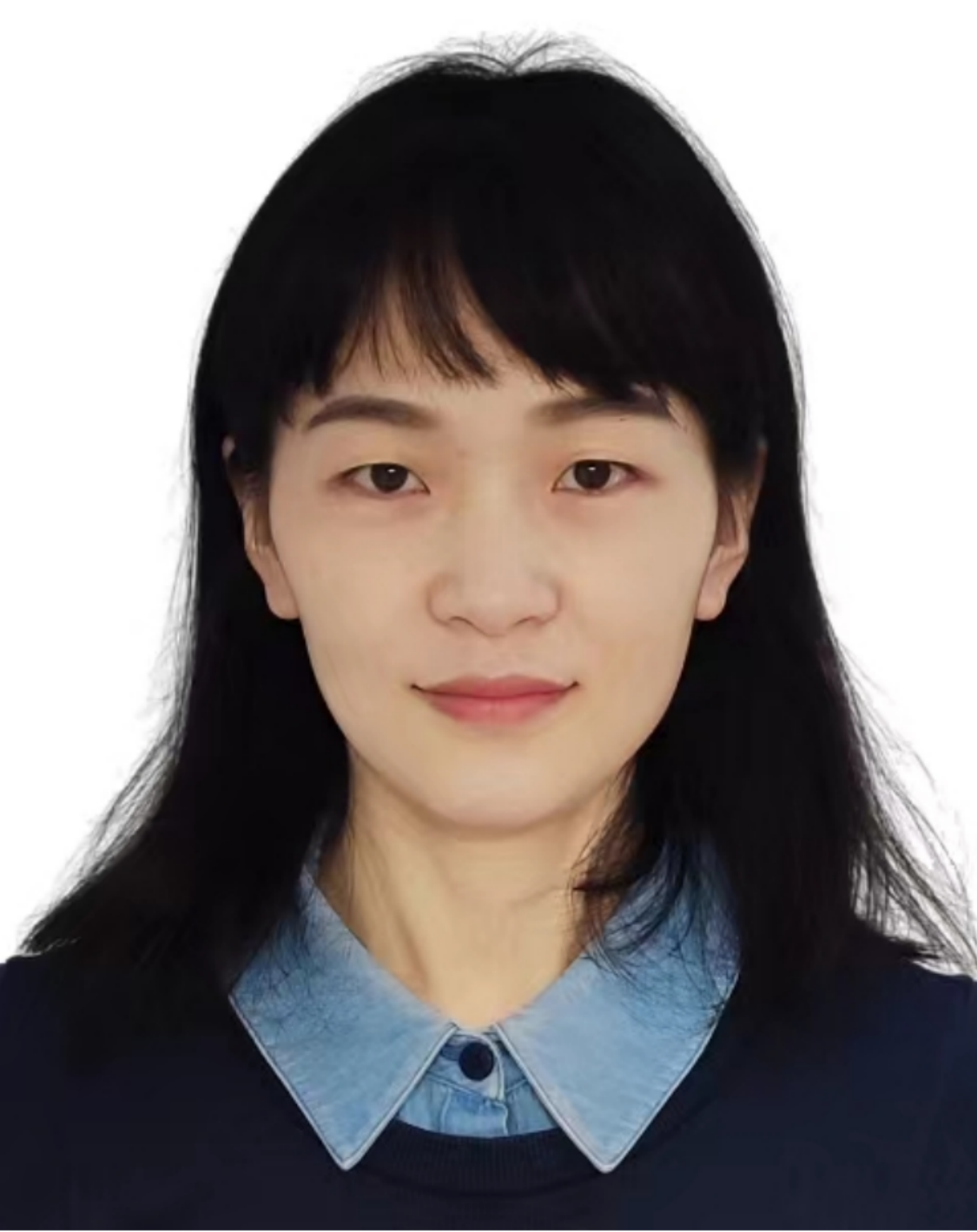}}]{Ce Zheng} is currently pursuing a Ph.D. at Beihang University. She graduated from the University of Chinese Academy of Sciences with a master’s degree. Her research interests include data mining, network measurements, video compression, and large language models.
\end{IEEEbiography}

\begin{IEEEbiography}[{\includegraphics[width=1in,height=1.25in,clip,keepaspectratio]{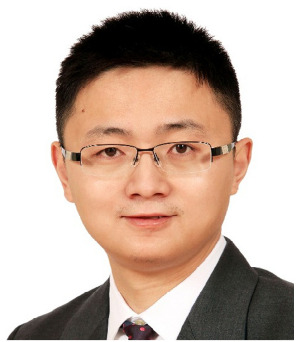}}]{Mai Xu}(Senior Member, IEEE) received the B.S. from Beihang University, Beijing, China, in 2003, an M.S. from Tsinghua University, Beijing, China, in 2006, and a Ph.D. from Imperial College London, London, U.K., in 2010. From 2010 to 2012, he was a Research Fellow with the Department of Electrical Engineering, Tsinghua University. Since January 2013, he has been with Beihang University, where he was an Associate Professor and was promoted to Full Professor in 2019. From 2014 to 2015, he was a Visiting Researcher at MSRA. His main research interests include image processing and computer vision. He has authored or coauthored more than 200 technical papers in international journals and conference proceedings, e.g., IJCV, IEEE TPAMI, TIP, J-STSP, CVPR, ICCV, ECCV, and AAAI. He is the recipient of the best/top paper awards of IEEE/ACM conferences, such as ACM MM. He served as an Associate Editor of IEEE TIP and IEEE TMM, a Lead Guest Editor of IEEE J-STSP, and an Area Chair or TPC Member for many conferences, such as ICME, AAAI, etc. He received outstanding AE awards twice (Years 2021 and 2022). He is an elected member of the Multimedia Signal Processing Technical Committee, IEEE Signal Processing Society.
\end{IEEEbiography}

\begin{IEEEbiography}[{\includegraphics[width=1in,height=1.25in,clip,keepaspectratio]{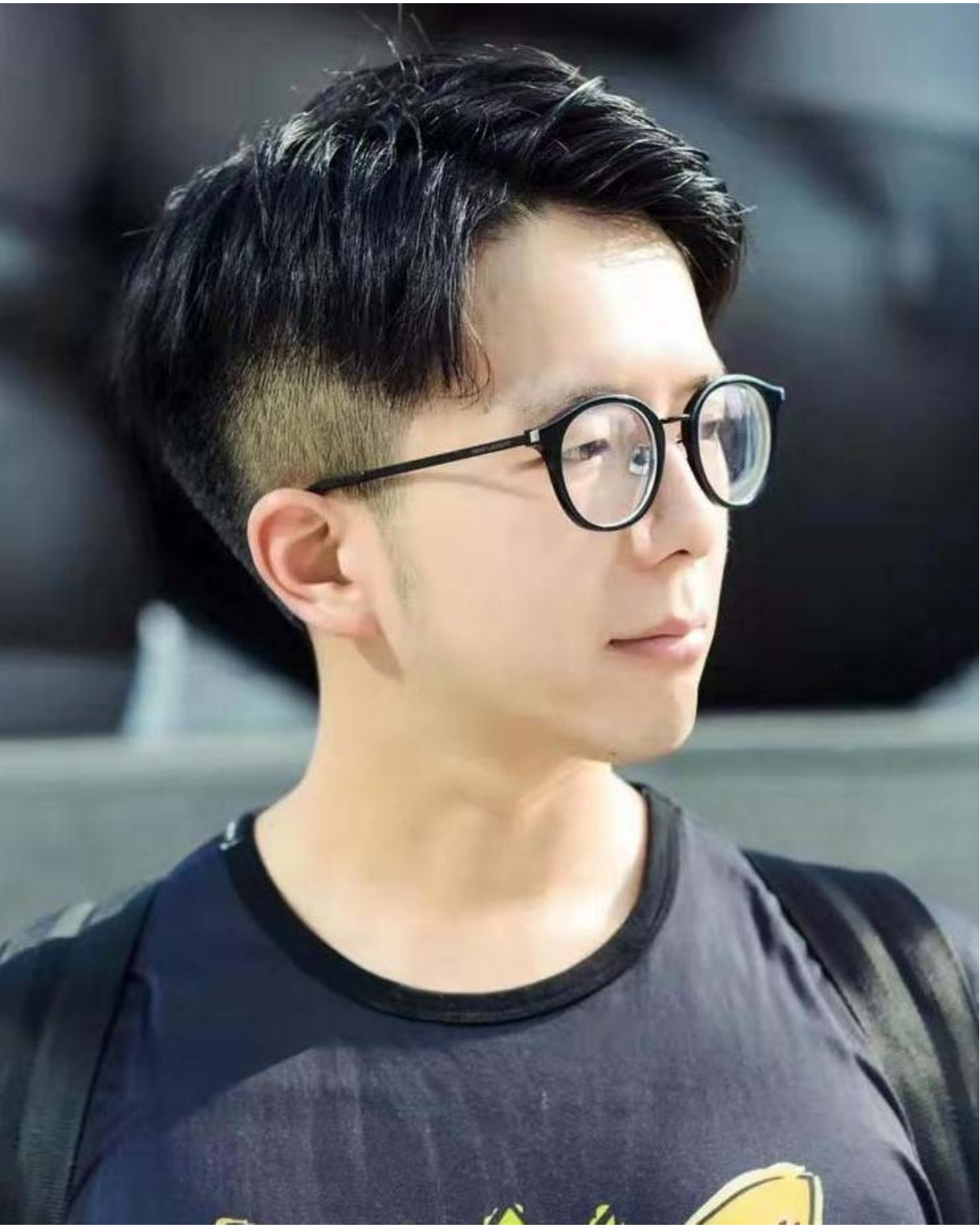}}]{Jingyuan Xia}(Member, IEEE)
is currently an Associate Professor with the College of the Electronic Science, the National University of Defense Technology (NUDT). His current research interests include low-level image processing, nonconvex optimization, and machine learning for signal processing.
\end{IEEEbiography}

\begin{IEEEbiography}
[{\includegraphics[width=1in,height=1.25in,clip,keepaspectratio]{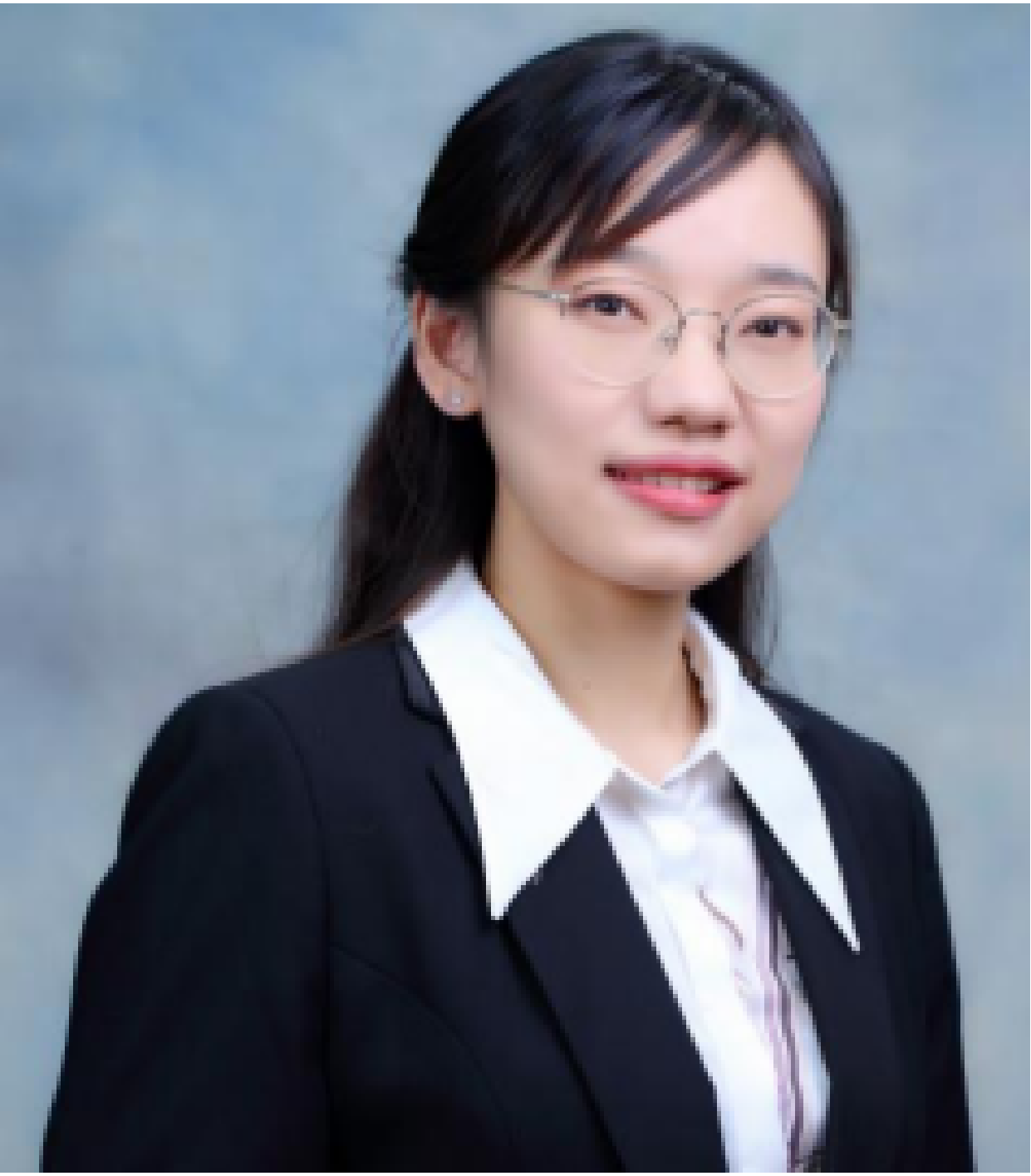}}]{Si Liu}(Member, IEEE) is currently a professor at Beihang University. She received her Ph.D. degree from Institute of
Automation, Chinese Academy of Sciences. She has been a research assistant and Postdoc in National
University of Singapore. Her research interest includes computer vision and multimedia analysis. She has published over 70 cutting-edge papers on human-related analysis and vision-language understating. She was the recipient of Best Paper Award of ACM MM 2021 and 2013, Best Demo Award of ACM MM 2012. She was the Champion of CVPR 2017 Look Into Person Challenge and the organizer of the ECCV 2018, ICCV 2019, CVPR 2021, CVPR 2022 and ACM MM 2022 Person in Context Challenges. She is the Associate Editor of IEEE TMM and TCSVT.
\end{IEEEbiography}
\vfill

\end{document}